%% file: draft-first-version.tex
\documentclass[11pt]{article}
\usepackage[inline]{enumitem}
\input{Begin0}

\input{Core}

\usepackage[left=3.5cm,right=3.5cm,top=4cm, bottom=4.5cm]{geometry}
\usepackage[utf8]{inputenc}
\usepackage[T1]{fontenc}
\usepackage{setspace}
\title{Error Analysis of Neural-Network-Based Engression}
\author{
Juntong Chen$^1$, \quad
Zijian Guo$^{2}$\thanks{Zijian Guo and Xinwei Shen contributed equally to this work and serve as co-corresponding authors.}, \quad
Xinwei Shen$^{3}$\footnotemark[1] \\[0.5cm]
$^1$School of Mathematical Sciences, Xiamen University \\
$^2$Center for Data Science, Zhejiang University\\
$^3$Department of Statistics, University of Washington
}
\date{}
\usepackage[utf8]{inputenc}
\usepackage[T1]{fontenc}
\usepackage{yhmath}
\usepackage{amsthm,amsmath,amssymb,amsbsy,bbm,mathrsfs,supertabular,
eurosym,graphicx,enumitem,xcolor}
\usepackage{bm}
\usepackage{mathtools}
\RequirePackage[colorlinks,citecolor=blue,urlcolor=blue]{hyperref}

\newcommand{\N}{\mathbb{N}}
\newcommand{\E}{\mathbb{E}}
\renewcommand{\P}{\mathbb{P}}

\newcommand{\coverN}{\mathcal{N}}

\usepackage[normalem]{ulem}
\usepackage[authoryear]{natbib}

\def\bbE{\mathbb{E}}

\def\cF{\mathcal{F}}
\def\cL{\mathcal{L}}
\def\cN{\mathcal{N}}
\def\cX{\mathcal{X}}

\theoremstyle{plain}
\newtheorem{theorem}{Theorem}
\newtheorem{proposition}[theorem]{Proposition}
\newtheorem{lemma}[theorem]{Lemma}

\theoremstyle{remark}

\begin{document}
\maketitle
\begin{abstract}
Engression \citep{shen2024engression} learns a conditional distribution by fitting a generative model $Y = f(X,\varepsilon)$ under the energy score, a strictly proper scoring rule. We provide a theoretical error analysis of engression implemented with deep neural networks. We decompose the excess risk into three components: the approximation error, the stochastic error, and the Monte Carlo error. Based on this decomposition, we establish convergence rates under the assumption that the target conditional generator admits a compositional smoothness structure.
\end{abstract}
\section{Introduction}
Given response variables $Y\in\R^p$ and a set of covariates or predictors $X\in\R^{d}$, a fundamental problem in statistics and machine learning is to model the dependence between $Y$ and $X$. Traditional regression paradigms predominantly focus on estimating the conditional mean or median. However, when the underlying conditional distribution exhibits multimodality, heteroscedasticity, or heavy-tailed asymmetry, these summary statistics are no longer adequate to capture the full spectrum of the relationship. To comprehensively understand how the response depends on the predictor, it becomes necessary to learn the entire conditional distribution $\P_{Y|X}$ over the $p$-dimensional sample space.

To this end, established statistical approaches provide foundational frameworks, including kernel conditional density estimation \citep{hall2004cross} and conditional quantile regression \citep{koenker2005quantile}. Subsequent work has derived theoretical guarantees for various nonparametric estimators; for example, \cite{pic2023} analyzed the $k$-NN and kernel methods, while \cite{peter2025} investigated the theoretical performance of local polynomial estimators. While these paradigms enjoy strong theoretical guarantees in low-dimensional settings, contemporary scientific problems, including representation learning \citep{bengio2013representation}, sufficient dimension reduction \citep{li1991sliced,cook1998regression}, and generative modeling, often arise in high-dimensional regimes where either $d$ or $p$ is large. This modern shift presents important opportunities to complement classical statistical tools with deep generative architectures, leveraging their expressive capacity to adaptively learn conditional distributions supported on nonlinear manifolds.

Recently, \cite{zhou2023deep} introduced a nonparametric generative approach for conditional distribution learning. Specifically, for a given predictor value $X=x$, the goal is to estimate a function $f^*(x,\varepsilon)$, where $\varepsilon$ is drawn from a simple reference distribution, such that the distribution of $f^*(x,\varepsilon)$ coincides with the conditional distribution of $Y$ given $X=x$. Building on this framework, \cite{song2026wasserstein} proposed augmenting the 1-Wasserstein objective with a regression regularization term, thereby jointly minimizing the distribution mismatch and the prediction error. Both methods are based on generative adversarial networks (GANs) \citep{goodfellow2014generative}, whose training is formulated as a min--max optimization problem that seeks a saddle point. Consequently, the optimization procedure can be computationally demanding and prone to instability.

\cite{shen2024engression} addressed this problem by leveraging an energy loss formulation and proposed a new framework, referred to as {\textit{engression}}. Unlike GAN-based approaches that require solving a saddle-point problem, engression only involves a standard minimization problem, leading to a more stable optimization procedure. Moreover, they demonstrated that engression exhibits superior extrapolation performance. Nevertheless, a finite-sample analysis of engression under deep neural network parameterizations remains limited. This paper addresses this problem. We also note the concurrent work of \citet{michigan2026}, which develops a broader theoretical analysis covering both engression and reverse Markov engression, while its single-step engression result is established over a standard H\"older smoothness class. When our compositional class specializes to the classical $\alpha$-H\"older smoothness class, our convergence rate has the same polynomial order as theirs, up to logarithmic factors, although the underlying assumptions differ. Another recent study by \citet{chai2026} also establishes convergence rates for engression under a hierarchical compositional structure, together with additional regularity conditions on the conditional distribution.

The article is structured as follows. In Section~\ref{setting-sec}, we introduce the statistical framework and some preliminaries of engression. Section \ref{dnns-section} presents the architecture of sparse deep neural networks and their fundamental properties. Under a compositional smoothness assumption on the true conditional generator, we establish non-asymptotic convergence rates for the expected excess risk in Section~\ref{rate}. The analysis is based on an error decomposition result, which is further elaborated in Section~\ref{rate}. Section~\ref{conclusion} presents the conclusions, while all proofs and extensions are deferred to Section~\ref{proofs-sec}. 

\textit{Notation}: Let $\mathbb{N}_0=\{0,1,\ldots\}$, $\mathbb{N}=\{1,2,\ldots\}$, and $\mathbb{R}_+=(0,\infty)$. For any $a \in \mathbb{R}$, $\lfloor a \rfloor$ denotes the largest integer strictly less than $a$. For any vector ${\bf z}=(z_1,\ldots,z_p)^\top\in\mathbb R^p$, define $\|\mathbf{z}\|_{2}
=\sqrt{(\sum_{i=1}^p z_i^2)/p}$, $\|\mathbf{z}\|_{\infty}
=
\max_{1\leq i\leq p}|z_i|$. For a vector-valued function
$f=(f_1,\ldots,f_p)^\top:\cD\to\mathbb R^p$, define $\|f\|_{L^\infty(\cD)}
=
\max_{1\leq j\leq p}\sup_{{\bs x}\in\cD}|f_j({\bs x})|$. When $\cD$ is clear from the context, we use the shorthand notation $\|\cdot\|_\infty$. For a $d_1 \times d_2$ matrix $M$, we define the entry-wise maximum norm $\|M\|_{\infty}=\max_{i\in[d_1],j\in[d_2]}|M_{i,j}|.$

\section{Statistical setting}\label{setting-sec}
In this paper, we consider the problem of conditional generative modeling. Specifically, let $\mathcal{D}_n = \{(X_i, Y_i)\}_{i=1}^n$ denote a dataset consisting of $n$ i.i.d. observations drawn from the product space $\mathcal{X} \times \mathcal{Y} \subseteq \mathbb{R}^{d} \times \mathbb{R}^{p}$. We assume that each observation pair is generated according to an underlying structural model defined by a measurable mapping $f^*:\mathcal{X} \times \mathcal{E} \rightarrow \mathcal{Y}$, such that
\begin{equation}\label{eq:statistical-setting}
Y_i = f^*(X_i, \varepsilon_i),
\end{equation}
where $\cE\subseteq\R^k$ denotes the noise space, and $\{\varepsilon_i\}_{i=1}^{n}$ are independent noise variables drawn from a known reference distribution $\mathbb{P}_{\varepsilon}$ (e.g., the uniform distribution), which are further assumed to be independent of the covariates $\{X_i\}_{i=1}^n$. For simplicity, we take $\cX\times\cE=[0,1]^{d+k}$ and assume that the distribution of $(X,\varepsilon)$ is supported
on this domain.

The generative formulation in \eqref{eq:statistical-setting} is motivated by the noise-outsourcing lemma stated as below. This result can be found, for example, in Theorem 5.10 of \cite{kallenberg} or Lemma 3.1 of \cite{austin2015}. We state it below for completeness.
\begin{lemma}\label{outsourcing-lemma}
Fix standard Borel spaces $\cX$, $\cY$ and let $X\in\cX$, $Y\in\cY$ be random vectors. Then there exists a measurable function $f^*:\cX\times[0,1]\rightarrow\cY$ and a random variable $\varepsilon\sim\operatorname{Unif}[0,1]$ such that $\varepsilon$ is independent of $X$ and $$(X,Y)=\big(X,f^*(X,\varepsilon)\big),\quad \mbox{almost surely}.$$  
\end{lemma}
The requirement that the underlying spaces be standard Borel is mild and covers a broad range of commonly encountered settings. In practice, one may also draw $\varepsilon$ from alternative reference distributions $\mathbb{P}_{\varepsilon}$ on $\cE$ that are easy to sample, such as the Gaussian distribution on $\mathbb{R}^k$ (see, e.g., \cite{shen2024engression} and \cite{song2026wasserstein}). Corresponding extensions of the noise-outsourcing lemma to the Gaussian setting can be found in \cite{zhou2023deep} and \cite{sharma2023bayesian}. Essentially, Lemma~\ref{outsourcing-lemma} implies that for $\mathbb{P}_X$-almost every $x \in \mathcal{X}$,
\begin{equation}\label{con-generator}
f^*(x,\varepsilon) \sim \mathbb{P}_{Y \mid X = x},
\quad\mbox{where}\quad
\varepsilon \sim \mathbb{P}_{\varepsilon}. \end{equation}
Thus, for any fixed \(x\), the conditional distribution \(\mathbb{P}_{Y \mid X=x}\) can be obtained as the pushforward of \(\mathbb{P}_{\varepsilon}\) through the mapping \(f^*(x,\cdot)\). Such a function \(f^*\) need not be unique, even when the distribution of the noise \(\varepsilon\) is specified. Moreover, the dimension of the noise \(\varepsilon\) need not coincide with that of the responses \(Y\). In the sequel, we refer to \(f^*\) in \eqref{eq:statistical-setting} as the {\em conditional generator}.

The goal is to recover the conditional distribution \(\mathbb{P}_{Y \mid X=x}\) induced by the model \eqref{eq:statistical-setting}. We assume that the true generative mapping \(f^*\) can be well approximated by a pre-specified function class \(\mathcal{F}\). In this work, \(\mathcal{F}\) is instantiated as a class of deep neural networks (see Section~\ref{dnns-section} for details), and the estimator \(\widehat{f} \in \mathcal{F}\) is obtained by minimizing an appropriate empirical loss, defined as follows.

To evaluate the performance of a given estimator, we employ the \textit{energy loss} introduced in the \textit{engression} framework \citep{shen2024engression}. This loss serves as a proper scoring rule derived from energy distance theory. For a data pair $(x, y) \in \mathcal{X} \times \mathcal{Y}$ and a candidate generator $f: \mathcal{X} \times \mathcal{E} \to \mathcal{Y}$, we define
\begin{equation}\label{L-exp-epsilon}
\mathcal{L}(f; x, y) = \mathbb{E}_{\varepsilon}\left[ \|y - f(x, \varepsilon)\|_2 \right] - \frac{1}{2} \mathbb{E}_{\varepsilon, \varepsilon'}\left[ \|f(x, \varepsilon) - f(x, \varepsilon')\|_2 \right],
\end{equation}
where $\varepsilon$ and $\varepsilon'$ are i.i.d. draws from the noise distribution \(\mathbb{P}_{\varepsilon}\). In the idealized design, the resulting estimator $\widehat{f}_{\infty}$ is defined as the empirical minimizer of the loss,
\begin{equation}\label{loss-1}
\widehat{f}_{\infty}\in \argmin_{f \in \mathcal{F}} \frac{1}{n} \sum_{i=1}^{n} \mathcal{L}(f; X_i, Y_i).
\end{equation}
In practical applications, however, evaluating the expectation over the noise distribution $\varepsilon$ in \eqref{L-exp-epsilon} can be computationally expensive. Instead, one commonly employs an empirical approximation by drawing \(m\) i.i.d. noise samples \(\{\varepsilon_j\}_{j=1}^m\) from \(\mathbb{P}_{\varepsilon}\). Denoting ${\bm\varepsilon}=(\varepsilon_1,\ldots,\varepsilon_m)^{\top}$, where $m\geq 2$, this leads to the following $m$-sample Monte Carlo surrogate of the loss in \eqref{L-exp-epsilon},
\begin{equation}\label{fm-loss-def}
\widehat\cL_m(f;x,y,{\bm\varepsilon}) 
= \frac{1}{m}\sum_{j=1}^{m}\|y - f(x,\varepsilon_j)\|_2 
- \frac{1}{2m(m-1)}\sum_{j \neq j'}^{m}\|f(x,\varepsilon_j) - f(x,\varepsilon_{j'})\|_2,    
\end{equation}
which yields the estimator
\begin{equation}\label{loss-2}
\widehat f_{m}\in\argmin_{f\in\cF}\frac{1}{n}\sum_{i=1}^{n}\widehat{\mathcal{L}}_m\big(f;X_i,Y_i,{\bm\varepsilon}^{(i)}\big).   \end{equation}
Here, each \(\bs{\varepsilon}^{(i)} = (\varepsilon^{(i)}_{1}, \ldots, \varepsilon^{(i)}_{m})^{\top} \in \mathcal{E}^m\) is a vector consisting of $m$ i.i.d. noise samples. 

At the population level, both \eqref{loss-1} and \eqref{loss-2} are intended to approximate the population risk $$\mathcal{R}(f)= \mathbb{E}[\mathcal{L}(f; X, Y)],$$ where the expectation is taken with respect to the joint distribution of \((X, Y)\). In the following, we focus on analyzing the theoretical performance of the Monte Carlo surrogate \eqref{loss-2}. The performance of the estimator in \eqref{loss-1} can be analyzed similarly; details are provided in Section~\ref{proofs-sec}. Heuristically, the objective in \eqref{loss-2} converges to that in \eqref{loss-1} as \(m \to \infty\), thereby linking the corresponding estimators. 

The performance of any estimator \(f:\mathcal{X}\times\mathcal{E}\rightarrow\mathcal{Y}\) is quantified by the \textit{excess risk}, defined as
\begin{equation}\label{def-excess-risk}
\mathcal{R}(f) - \mathcal{R}(f^*)
=
\mathbb{E}[\mathcal{L}(f; X, Y)]
-
\mathbb{E}[\mathcal{L}(f^*; X, Y)],
\end{equation}
where \(f^*\) denotes the true conditional generator in \eqref{eq:statistical-setting}. The excess risk associated with the energy loss in \eqref{def-excess-risk} admits a natural interpretation as a grounded distributional discrepancy. In the univariate setting ($p=1$), so that $Y\in\R$, the excess risk is equivalent to the integrated squared difference between the corresponding cumulative distribution functions (CDFs) \citep{baringhaus2004new,szekely2013energy}, known as the squared $L_2$-Cramér distance
\begin{equation}\label{excess-to-cramer}
\mathcal{R}(f) - \mathcal{R}(f^*) = \mathbb{E}_X\left[\int_{\mathbb{R}}\big(F_f(y|X) - F_{f^*}(y|X)\big)^2 \, dy\right],
\end{equation}
where $F_{f^*}(y|x) = \mathbb{P}_{\varepsilon}(f^*(x,\varepsilon) \leq y)$ denotes the true conditional CDF, and $F_f(y|x) = \mathbb{P}_{\varepsilon}(f(x,\varepsilon) \leq y)$ is the conditional CDF induced by the candidate generator $f$. For the multivariate case ($p > 1$), the excess risk generalizes via the energy distance. Recall that the squared energy distance between two distributions $P$ and $Q$ is defined as
\begin{equation}
D_E^2(P,Q) =\sqrt{p}\cro{2\mathbb{E}\|Y - Z\|_2 - \mathbb{E}\|Y - Y'\|_2 - \mathbb{E}\|Z - Z'\|_2},
\end{equation}
where $Y, Y' \sim P$ and $Z, Z' \sim Q$ are independent copies. Let $\mathbb{P}_{f|X=x}$ denote the conditional distribution of $Y$ given $X=x$, induced by the mapping $f(x, \varepsilon)$ under the noise distribution $\mathbb{P}_{\varepsilon}$. It then follows that
\begin{equation}\label{connect-energy-distance}
\mathcal{R}(f) - \mathcal{R}(f^*) = \frac{1}{2\sqrt{p}} \mathbb{E}_X\left[ D_E^2\big.(\mathbb{P}_{f^*|X}, \mathbb{P}_{f|X}\big) \right],
\end{equation}
which shows that the excess risk is proportional to the expected squared energy distance between the oracle and candidate conditional distributions.

\section{Deep neural network classes}\label{dnns-section}
The proposed estimation framework is built upon a class of deep ReLU neural networks. This section details their architectural configurations and states their fundamental properties relevant to the subsequent error analysis.

Let $\bs{p} = (p_0, p_1, \ldots, p_{L+1}) \in \N^{L+2}$ denote a vector, where $p_0 = d+k$ and $p_{L+1}=p$. The rectified linear unit (ReLU) activation function is defined as $$\operatorname{ReLU}(u) = \max\{u,0\},\quad u\in\R.$$ We define $\cF{(L,\bs{p})}$ as the class of deep ReLU neural networks with depth $L \in \N_0$ and width vector $\bs{p}$, consisting of all functions $f:\R^{d+k}\rightarrow\R^{p}$ of the form
\begin{equation}\label{neural-network}
f({\bm{x}})=A_{L}\circ\operatorname{ReLU}\circ \hspace{2pt}A_{L-1}\circ\cdots\circ\operatorname{ReLU}\circ\hspace{2pt}A_{0}({\bs{x}}),
\end{equation}
where, for $\ell=0,\ldots,L$, $$A_{\ell}({\bs{y}})=M_{\ell}{\bs{y}}+b_{\ell}.$$ Here, $M_{\ell}$ is a $p_{\ell+1}\times p_{\ell}$ weight matrix, $b_{\ell}$ is a bias vector of size $p_{\ell+1}$, and the $\operatorname{ReLU}$ activation function is applied component-wise to any given vector. Throughout this paper, we assume that all entries of the weight matrices $M_{\ell}$ and bias vectors $b_{\ell}$ are bounded within $[-1,1]$. In practice, sparsity in neural networks is often encouraged through techniques such as regularization or specialized architectures \citep{goodfellow2016deep}. A notable example is dropout, which promotes sparse activation patterns by randomly deactivating units during training, thereby ensuring that each neuron is active only for a small subset of the training data \citep{2014dropout}. In line with the framework introduced by \cite{Schmidt-Hieber}, we explicitly enforce parameter sparsity by restricting the network to utilize only a limited number of non-zero parameters. Specifically, let $\|M_{\ell}\|_0$ and $|b_{\ell}|_0$ denote the number of non-zero entries in the weight matrix $M_{\ell}$ and bias vector $b_{\ell}$, respectively. For any $\cB> 0$, let $\sigma_\cB: \mathbb{R}^p \to [-\cB,\cB]^p$ denote the component-wise truncation operator defined as 
\[
\sigma_{\cB}({\bs x}) = \big((x_1 \vee (-\cB)) \wedge \cB, \ldots, (x_p \vee (-\cB)) \wedge \cB\big)^\top,
\]
for ${\bs x}= (x_1, \ldots, x_p)^\top \in \mathbb{R}^p$. Then, the class of $s$-sparse deep ReLU networks restricted to $[-\cB,\cB]^p$ is defined as
\begin{equation*}
\cF(L,{\bs p},s,\cB) = \left\{ \sigma_{\cB} \circ f : f \in \cF(L,{\bs p}), \, \sum_{\ell=0}^L \left( \|M_{\ell}\|_0 + |b_{\ell}|_0 \right) \leq s \right\}.
\end{equation*}
When truncation to $[-\cB,\cB]$ is not required in a given context, we write $\cF(L,{\bs p},s,\infty)$ for the corresponding class. Figure~\ref{fig:dnn-graph} provides a schematic illustration of the network architecture. As mentioned in Section~\ref{setting-sec}, to estimate the conditional distribution $\mathbb{P}_{f^*|X}$, we consider a suitable network class $\cF(L,{\bs p},s,\cB)$ with appropriately chosen depth $L$, width ${\bs p}$, and sparsity level $s$ for the true conditional generator $f^*$, and seek a good candidate within the class. In what follows, we use $\cF_{\rm NN}$ to denote a generic network
class of the form $\cF(L,{\bs p},s,\cB)$, for simplicity.

It is worth emphasizing that, for a given network class $\cF_{\operatorname{NN}}$, we do not assume that $f^*\in\cF_{\operatorname{NN}}$. In the misspecified setting where $f^*\notin\cF_{\operatorname{NN}}$, it follows from the characterization \eqref{connect-energy-distance} and the triangle inequality of the energy distance that, for any $f\in\cF_{\operatorname{NN}}$,
\begin{equation}\label{decom-dis}
\cR(f)-\cR(f^*)\leq \frac{1}{\sqrt{p}}\cro{\mathbb{E}_X D_E^2\left(\mathbb{P}_{f^*|X}, \mathbb{P}_{\bar f|X}\right)+ \mathbb{E}_XD_E^2\left(\mathbb{P}_{\bar f|X}, \mathbb{P}_{f|X}\right)},
\end{equation}
where $\bar f$ is the best approximator among $\cF_{\operatorname{NN}}$, in the sense that its induced conditional distribution $\P_{\bar f|X}$ minimizes the expected squared energy distance to the true conditional distribution $\P_{f^*|X}$ over $\cF_{\operatorname{NN}}$.
\begin{figure}[ht]
\centering
\includegraphics[width=0.5\linewidth]{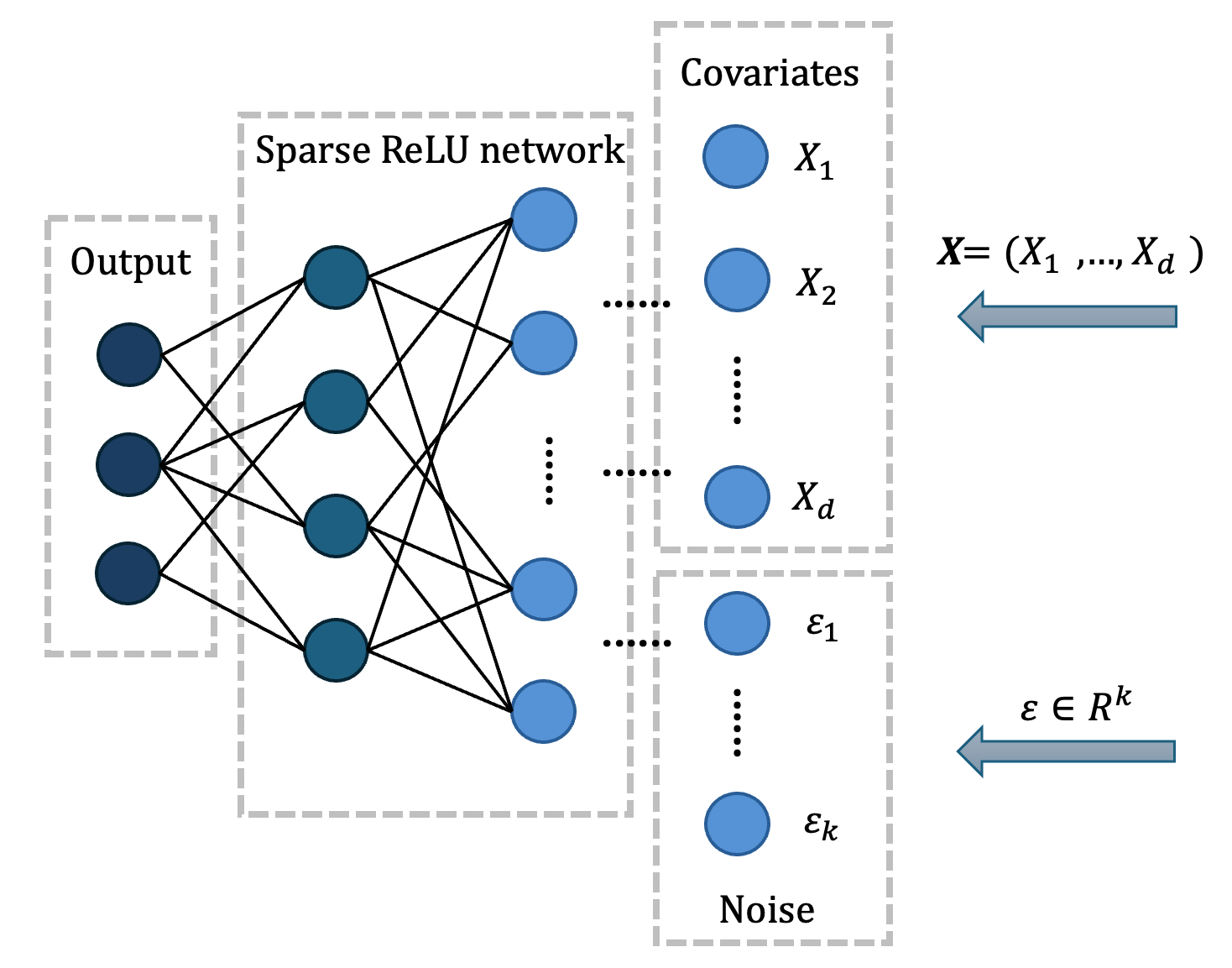}
\caption{A schematic illustration of the deep sparse ReLU network, taking both covariates and noise variables as inputs and producing a $p$-dimensional output vector.}
\label{fig:dnn-graph}
\end{figure}

In particular, \eqref{decom-dis} highlights the importance of choosing an appropriate neural network class $\cF(L,{\bs p},s,\cB)$ for the implementation of \eqref{loss-2}, with a balance between approximation capacity and model complexity. To quantify the complexity of the vector-valued function class $\cF$, we employ its metric entropy. For any two functions $f = (f_1, \ldots, f_p)^\top$ and $g = (g_1, \ldots, g_p)^\top$ mapping from $\mathcal{X}\times\cE$ to $\cY\subseteq \mathbb{R}^p$, define the $L^\infty$-distance (or sup-norm distance) as
\[
\|f - g\|_{L^\infty(\mathcal{X}\times\cE)} = \max_{1 \leq j \leq p} \|f_j - g_j\|_{L^\infty(\mathcal{X}\times\cE)}.
\]
For any $\delta>0$, a set $\mathcal{F}_{\delta}$ is called a $\delta$-cover of $\mathcal{F}$ with respect to the $L^{\infty}$-norm if, for every $f\in\mathcal{F}$, there exists a function $g\in\mathcal{F}_{\delta}$ such that $$\|f - g\|_{L^\infty(\mathcal{X}\times\cE)}\leq\delta.$$ The $\delta$-covering number of $\cF$, denoted by $\mathcal{N}(\delta, \mathcal{F},\|\cdot\|_{L^\infty(\mathcal{X}\times\cE)})$, is defined as the minimum cardinality among all valid $\delta$-covers of $\cF$. If no finite $\delta$-cover exists, the covering number is taken to be infinite. The metric entropy (or simply entropy) of $\mathcal{F}$ at scale $\delta>0$ is then formally defined as $\log\mathcal{N}\big(\delta, \mathcal{F}, \|\cdot\|_{L^\infty(\mathcal{X}\times\cE)}\big)$. The following result provides an upper bound for the metric entropy of $\mathcal{F}(L,\bs{p},s,\cB)$.
\begin{proposition}\label{cover-bound}
Let $V=\prod_{\ell=0}^{L+1} (p_\ell + 1)$. For any $0<\delta\leq 1$,
\begin{align*}
&\log\mathcal{N}\big(\delta,\mathcal{F}(L, {\bs p}, s,\cB),\|\cdot\|_{L^\infty(\mathcal{X}\times\cE)}\big)\leq(s+1)\log\left(\frac{2(L+1)V^2}{\delta}\right).
\end{align*} 
\end{proposition}
\begin{proof}
The proof adapts the argument from Lemma 5 of \cite{Schmidt-Hieber} to accommodate $p$-dimensional outputs, leveraging the fact that $\left\|\sigma_{\cB}\circ f - \sigma_{\cB}\circ g \right\|_{L^\infty(\cX\times\cE)} \leq \|f-g\|_{L^\infty(\cX\times\cE)}$.
\end{proof}
Deeper and wider ReLU neural networks equipped with more active parameters naturally possess greater expressive power, thereby enabling a more accurate approximation of the target conditional generator $f^*$ (see Lemma~\ref{dnn-approx-lemma} for details). Concomitantly, as stated in Proposition~\ref{cover-bound}, such expansive network classes entail a larger entropy bound, leading to heightened model complexity. A central insight from statistical learning theory is that the stochastic component of the estimation error is governed crucially by this measure of capacity \citep{van-exp,Bartlett_2005,Schmidt-Hieber}.

\section{Convergence rates and error analysis}\label{rate}
In this section, we study the convergence rates of the expected excess risk under compositional smoothness assumptions on the underlying conditional generator $f^*$. Specifically, we suppose that $f^*$ admits the representation $f^* = f_q^* \circ f_{q-1}^* \circ \cdots \circ f_0^*$, where each $f_i^*=(f_{i,1}^*,\ldots,f_{i,d_{i+1}}^*)^\top$ maps $[a_i,b_i]^{d_i}$ to $[a_{i+1},b_{i+1}]^{d_{i+1}}$. Moreover, for every $j\in\{1,\ldots,d_{i+1}\}$, the component function $f_{i,j}^*$ is assumed to be $\alpha_i$-H\"older smooth and to depend on at most $t_i$ components. The above setting is sufficiently general to cover several classical models considered in the existing literature. For instance, if $q=0$ and $p=1$, it reduces to the standard H\"older smooth setting studied in \cite{Stone1982} and \cite{Tsybakov2009}; if $q=1$, $p=1$ and each $f_{0,j}^*$ depends only on $x_j$, and
$f_1^*({\bf z})=\sum_{j=1}^d z_j$, then the model reduces to the
standard additive model considered in \cite{HuangHorowitzWei2010}. These examples illustrate the broad applicability of the compositional framework. 

We now provide a precise definition of the target function class. Recall that for any $\alpha>0$, a function $f$ is said to be $\alpha$-H\"older smooth if all partial derivatives
up to order $\lfloor \alpha \rfloor$ exist and are bounded, and all partial derivatives of order $\lfloor \alpha \rfloor$
are $(\alpha-\lfloor \alpha \rfloor)$-H\"older continuous. For $t\in\mathbb{N}$ and $\alpha>0$, we define the $\alpha$-H\"older ball on $\mathcal{D}\subseteq\mathbb{R}^t$ with radius $K\geq0,$ denoted by $\cH_t^\alpha(\cD,K),$ as the class of functions $f:\cD\to\mathbb{R}$ satisfying
$$\sum_{\substack{{\bs{\beta}}=(\beta_{1},\ldots,\beta_{t})\in\N_0^{t}\\ \sum_{j=1}^{t}\beta_{j}<\alpha}}\|\partial^{\bs{\beta}}f\|_{\infty}+\sum_{\substack{{\bs{\beta}}\in\N_0^{t}\\ \sum_{j=1}^{t}\beta_{j}=\lfloor\alpha\rfloor}}\sup_{\substack{{\bs{x}},{\bs{y}}\in\cD\\{\bs{x}}\not={\bs{y}}}}\frac{\left|\partial^{\bs{\beta}}f({\bs{x}})-\partial^{\bs{\beta}}f({\bs{y}})\right|}{|{\bs{x}}-{\bs{y}}|_{\infty}^{\alpha-\lfloor\alpha\rfloor}}\leq K,$$ where for any multi-index ${\bs{\beta}}=(\beta_{1},\ldots,\beta_{t})\in\N_0^{t}$, we write $$\partial^{\bs{\beta}}=\partial^{\beta_{1}}\cdots\partial^{\beta_{t}}.$$

We assume that each constituent function satisfies $f^*_{i,j} \in \mathcal{H}^{\alpha_i}_{t_i}([a_i, b_i]^{t_i}, K)$. Accordingly, the underlying class of compositional functions is defined as
\begin{align}
\cG(q,{\bs{d}},\ &{\bs{t}},{\bm{\alpha}},K)=\left\{g_{q}\circ\cdots\circ g_{0}:\;g_{i}=(g_{i,j})_{j}:\cro{a_{i},b_{i}}^{d_{i}}\rightarrow\cro{a_{i+1},b_{i+1}}^{d_{i+1}},\right.\nonumber\\
&\quad\quad\quad\quad\quad\quad\quad\quad\left.g_{i,j}\in\cH^{\alpha_{i}}_{t_{i}}(\cro{a_{i},b_{i}}^{t_{i}},K),\;\mbox{for some } |a_{i}|, |b_{i}|\leq K\right\},\label{composite-def}
\end{align}
where $q\in\mathbb{N}_0$, ${\bs{d}}=(d_{0},\ldots,d_{q+1})\in\mathbb{N}^{q+2}$ with $d_{0}=d+k$ and $d_{q+1}=p$, ${\bs{t}}=(t_{0},\ldots,t_{q})\in\mathbb{N}^{q+1}$, ${\bm{\alpha}}=(\alpha_{0},\ldots,\alpha_{q})\in\R_{+}^{q+1}$, and $K\geq0$. Here, $t_i\leq d_i$ represents the effective input dimension of each component function $g_{i,j}$, $j=1,\ldots,d_{i+1}$, at the $i$-th layer.

For any function $f=f_{q}\circ\cdots\circ f_{0}\in\cG(q,{\bs{d}},{\bs{t}},{\bm{\alpha}},K)$, the overall smoothness of $f$ is governed by the smoothness of its constituent mappings. The classical $t$-variate H\"older class arises as a special case corresponding to $q=0$, $t_0=t$ and $p=1$. By contrast, when $q = 1$ with $\alpha_0,\alpha_1 \leq 1$ and $d_0 = d_1 = t_0 = t_1 =p= 1$, the composition $f = g_1 \circ g_0$ possesses smoothness of order $\alpha_0\alpha_1$, a phenomenon well documented in the literature \citep{2017regularity,juditsky2009nonparametric,baraud2014estimating}. To capture the smoothness structure of more general compositional hierarchies, define the effective smoothness parameters, for $i=0,\ldots,q-1$, by
\begin{equation}\label{effect-smooth}
\alpha_i^*=\alpha_i \prod_{\ell=i+1}^q (\alpha_\ell \wedge 1),    
\end{equation}
with $\alpha_q^*=\alpha_q$. These effective smoothness indices play a central role in characterizing the convergence rates of the proposed network estimator.

In fact, sparse deep ReLU networks with suitably chosen architectures can approximate the compositional H\"older class $\cG(q,{\bs{d}},{\bs{t}},{\bm{\alpha}},K)$ arbitrarily well.
\begin{lemma}\label{dnn-approx-lemma}
Let $t_i \in \mathbb{N}$ and $\alpha_i\in\R_{+}$, for $i = 0, \ldots, q$ and let $\alpha_i^*$ be defined as in \eqref{effect-smooth}. Set $m_i=\lceil (\alpha_i+t_i)\log_2 n/(2\alpha_i^*+t_i) \rceil$, $L'_i = 8 + (m_i+5)(1 + \lceil \log_2(t_i \vee \alpha_i) \rceil)$, and $s_i= 141(t_i + \alpha_i + 1)^{3+t_i}(m_i + 6).$ Let $\varphi^{*} \in \cG(q, \bs{d}, \bs{t}, \bm{\alpha}, K)$ and let $Q_0=1$, $Q_i=(2K)^{\alpha_i}$ for $i \in [q-1]$, and $Q_q=K(2K)^{\alpha_q}$.
For any $N_i \in \mathbb{N}$ such that $N_i \geq (\alpha_i+1)^{t_i} \vee (Q_i+1)e^{t_i}$, there exists $h \in \cF\big(\overline{L}, (d+k, 6rN, \ldots, 6rN, p), s,\infty\big)$ with 
$$\overline L=3q + \sum_{i=0}^q L_i',\quad r= \max_{i=0,\ldots,q} d_{i+1}(t_i + \lceil \alpha_i \rceil),\quad N=\max_{i=0,\ldots,q}N_i,\quad s\leq\sum_{i=0}^qd_{i+1}(s_iN+4),$$
and a positive constant $C$ depending only on $q,{\bs d},{\bs t},{\bs\alpha},K$, such that for all $j=1,\ldots,p$,
\[
\sup_{{\bs y}\in[0,1]^{d+k}}|\varphi^{*}_j({\bs{y}}) - h_j({\bs{y}})|\leq C\cro{\sum_{i=0}^{q}\left(N_i^{-\frac{\alpha_i}{t_i}}+N_in^{-\frac{\alpha_i+t_i}{2\alpha_i^*+t_i}}\right)^{\prod_{\ell=i+1}^{q}(\alpha_{\ell}\wedge1)}}.
\] 
\end{lemma}
The above result follows from a slight modification of Lemma~13 in \cite{chen2026semi}, itself derived from Theorem~1 in \cite{Schmidt-Hieber} (page~1891), upon incorporating the corrections proposed in \cite{Vu2024Correction}. With Lemma~\ref{dnn-approx-lemma} in place, we establish the following convergence rate for the estimator defined in \eqref{loss-2} based on deep sparse neural networks.

\begin{theorem}\label{thm-rate}
Suppose $f^{*}\in\cG(q,{\bs{d}},{\bs{t}},{\bm{\alpha}},K)$ with $K\geq1$. Set $N_i=\lceil n^{t_i/(t_i+2\alpha_i^*)}\rceil$, $N=\max_{i=0,\ldots,q}N_i$, and let $\widehat f_{m}$ denote the estimator obtained from \eqref{loss-2} over the network class $\mathcal{F}(L_n,\bs{p}_n, s_n,\cB)$, whose architecture satisfies
\begin{enumerate}[label=(\roman*), itemsep=0pt, topsep=2pt, parsep=0pt]
\item $L_n=C_{1} \log_2 n$;
\item ${\bs{p}}_n=\left(d+k,C_{2}N,\ldots,C_{2}N,p\right)$;
\item $s_n=C_{3} N\log_2 n$;
\item $\cB\geq K$.
\end{enumerate}
Then, for all sufficiently large $n$, 
\begin{align*}
\mathbb{E}_{\cD_n,\cE_n}\bigl[\cR(\widehat f_m)-\cR(f^*)\bigr]\leq C_{4}\max_{i=0,\ldots,q}n^{-\frac{\alpha_i^*}{2\alpha_i^*+t_i}}\log^{3/2} n.
\end{align*} 
Here, $C_1,C_2,C_3$ and $C_4$ are positive numerical constants independent of $n$.
\end{theorem}
The proof of Theorem~\ref{thm-rate} is deferred to Section~\ref{proof-sec-decom}. Here, we briefly comment on the result. First, the convergence rate holds uniformly over all noise distributions supported on $[0,1]^k$. Second, the rate depends only on the intrinsic compositional structure of the conditional generator function, rather than the ambient dimension, thereby effectively reducing the curse of dimensionality. Ultimately, the overall rate is determined by the bottleneck \(\max_{i=0,\dots,q} n^{-\alpha_i^*/(2\alpha_i^*+t_i)}\), up to logarithmic factors. In particular, when the target conditional generator reduces to the $\alpha$-H\"older smoothness class, the rate obtained here agrees with that derived in Theorem~1 of \cite{michigan2026}, up to logarithmic factors, despite slight differences in the noise assumptions.

It is worth noting that the rate obtained here differs from that established in Corollary~4.2 of \cite{chai2026}. Among other regularity conditions, the authors assume that the conditional quantile function admits a hierarchical compositional structure as in \citet{sophie2021}. They further require the conditional distribution function to satisfy \begin{equation}\label{add-condition} \left| \P(Y\leq y_1\mid X=\bs x)-\P(Y\leq y_2\mid X=\bs x)\right| \leq c_0|y_1-y_2| \end{equation} for every $\bs x\in[0,1]^d$ and $y_1,y_2\in[-1,1]$, where $c_0>0$ is a constant. Thus, the conditional distribution function is required to be uniformly Lipschitz continuous in its response argument. Under these conditions, \citet{chai2026} derive a faster convergence rate. Condition~\eqref{add-condition}, in particular, rules out conditional
distributions with atoms, even when the underlying conditional generator is
smooth. For example, consider $Y=g(X)+\varepsilon$, $\varepsilon\sim\operatorname{Bernoulli}(1/2)$, where $g:[0,1]^d\to[-1/2,0]$ is a smooth function. This model admits the smooth conditional generator $f^*(\bs{x},\varepsilon)=g(\bs{x})+\varepsilon$. However, conditionally on $X=\bs{x}$, the response $Y$ takes the values
$g(\bs{x})$ and $g(\bs{x})+1$, each with probability $1/2$. Its conditional
distribution function therefore has jumps at these two points and fails to
satisfy the Lipschitz condition in \eqref{add-condition}.

\subsection{Error decomposition}
The derivation of Theorem~\ref{thm-rate} relies on an analysis of the error decomposition for $\widehat f_{m}$, which we present in detail in this subsection.

Recall that we denote the $n$ i.i.d. observations by $\mathcal{D}_n = \{Z_i\}_{i=1}^n$, where $Z_i=(X_i, Y_i)\in\cX\times\cY.$ Let $Z = (X, Y)$ denote a generic random observation drawn from the same distribution specified in \eqref{eq:statistical-setting}. For any measurable function $f:\cX\times\cE\rightarrow\cY$, define the population risk and its empirical counterpart, respectively, by $\mathcal{R}(f)= \mathbb{E}[\mathcal{L}(f; X,Y)]$, and $\mathcal{R}_n(f)= \frac{1}{n} \sum_{i=1}^n \mathcal{L}(f; X_i,Y_i).$ Additionally, define the population and empirical risks associated with the
$m$-sample loss by
\begin{align}
\mathcal{R}^{(m)}(f)=
\mathbb{E}_{Z}
\left\{
\frac{1}{n}\sum_{i=1}^n
\widehat{\mathcal{L}}_m
\big(f; Z,\bm{\varepsilon}^{(i)}\big)
\right\},\ 
\mathcal{R}_n^{(m)}(f)=
\frac{1}{n}\sum_{i=1}^n
\widehat{\mathcal{L}}_m
\big(f; X_i,Y_i,\bm{\varepsilon}^{(i)}\big).
\label{def-cRm}
\end{align}
Here
\[
\bm{\varepsilon}^{(i)}
=
(\varepsilon^{(i)}_{1}, \ldots, \varepsilon^{(i)}_{m})^\top .
\]
The collection
\[
\{\varepsilon_j^{(i)}:
i=1,\ldots,n,\ j=1,\ldots,m\}
\]
consists of mutually independent samples from $\P_\varepsilon$
and is independent of $\cD_n$. Observe that
$\mathcal{R}^{(m)}(f)$ remains random, since it depends on the noise vectors
$\bm{\varepsilon}^{(1)},\ldots,\bm{\varepsilon}^{(n)}$. 

Writing $\cE_n=(\bm{\varepsilon}^{(1)},\ldots,\bm{\varepsilon}^{(n)})$, it follows that for every fixed $f$
$$\mathcal{R}(f)=\E_{\cE_n}\cro{\mathcal{R}^{(m)}(f)}\quad\mbox{and}\quad\mathcal{R}(f)=\E_{\cD_n,\cE_n}\cro{\mathcal{R}_{n}^{(m)}(f)}.$$

We begin by presenting a fundamental error decomposition result for the excess risk of the estimator $\widehat f_m$.
\begin{lemma}
\label{prop:error-decomp-m}
Let $\widehat f_m$ be the estimator defined in \eqref{loss-2} over the class $\cF_{\operatorname{NN}}$ of functions mapping $\cX\times\cE$ to $\cY$. Then, for any $f \in \mathcal{F}_{\operatorname{NN}}$,
\begin{align*}
&\mathcal{R}(\widehat f_m) - \mathcal{R}(f^*)\\ &\leq\underbrace{\mathcal{R}(f) - \mathcal{R}(f^*)}_{\text{approximation error}}+\underbrace{2\sup_{g \in \mathcal{F}_{\operatorname{NN}}} \Bigl| \mathcal{R}^{(m)}(g) - \mathcal{R}_{n}^{(m)}(g) \Bigr|}_{\text{stochastic error}}+\underbrace{2\sup_{g \in \mathcal{F}_{\operatorname{NN}}}\Bigl| \mathcal{R}(g) - \mathcal{R}^{(m)}(g) \Bigr|.}_{\text{Monte Carlo error}}
\end{align*}
\end{lemma}
\begin{proof}
Observe that for any $f \in \mathcal{F}_{\operatorname{NN}}$,
\begin{align*}
&\mathcal{R}(\widehat f_m) - \mathcal{R}(f^*)\\
&= \Bigl(\mathcal{R}(\widehat f_m) - \mathcal{R}^{(m)}(\widehat f_m)\Bigr) 
   + \Bigl(\mathcal{R}^{(m)}(\widehat f_m) - \mathcal{R}_{n}^{(m)}(\widehat f_m)\Bigr)\\
&\quad + \Bigl(\mathcal{R}_{n}^{(m)}(\widehat f_m) - \mathcal{R}_{n}^{(m)}(f)\Bigr) 
   + \Bigl(\mathcal{R}_{n}^{(m)}(f) - \mathcal{R}^{(m)}(f)\Bigr)\\
&\quad + \Bigl(\mathcal{R}^{(m)}(f) - \mathcal{R}(f)\Bigr) 
   + \Bigl(\mathcal{R}(f) - \mathcal{R}(f^*)\Bigr). 
\end{align*}
Since $\widehat{f}_m$ minimizes the empirical risk $\mathcal{R}_{n}^{(m)}$ over $\cF_{\operatorname{NN}}$, the difference $\mathcal{R}_{n}^{(m)}(\widehat{f}_m) - \mathcal{R}_{n}^{(m)}(f)$ is non-positive for any $f \in \mathcal{F}_{\operatorname{NN}}$. The conclusion then follows by applying uniform bounds to the terms in the decomposition
\begin{equation*}
\Bigl(\mathcal{R}(\widehat f_m) - \mathcal{R}^{(m)}(\widehat f_m)\Bigr)+\Bigl(\mathcal{R}^{(m)}(f) - \mathcal{R}(f)\Bigr)\leq 2\sup_{g \in \mathcal{F}_{\operatorname{NN}}}\bigl| \mathcal{R}(g) - \mathcal{R}^{(m)}(g) \Bigr|,   
\end{equation*}
\begin{equation*}
\Bigl(\mathcal{R}^{(m)}(\widehat f_m) - \mathcal{R}_{n}^{(m)}(\widehat f_m)\Bigr)+\Bigl(\mathcal{R}_{n}^{(m)}(f) - \mathcal{R}^{(m)}(f)\Bigr)\leq 2\sup_{g \in \mathcal{F}_{\operatorname{NN}}}\Bigl| \mathcal{R}^{(m)}(g) - \mathcal{R}^{(m)}_n(g) \Bigr|.   
\end{equation*}
\end{proof}
Lemma~\ref{prop:error-decomp-m} shows that the excess risk of $\widehat f_m$ is decomposed into three components: (i) the approximation bias, which arises from the expressive capacity of the neural network class $\cF_{\operatorname{NN}}$; (ii) the stochastic error, which stems from the finite sample size $n$ and is evaluated through the lens of the Monte Carlo surrogate loss; (iii) the Monte Carlo error, representing the discrepancy introduced by approximating the population expectation with a finite average over $m$ noise samples. Notably, we do not require the true conditional generator $f^*$ to belong to $\cF_{\operatorname{NN}}$. In this misspecified setting, the approximation error is controlled by $\inf_{f\in\cF_{\operatorname{NN}}}
\bigl\{\cR(f)-\cR(f^*)\bigr\}$.

The following result establishes a Lipschitz-type property of the loss function, which facilitates the control of the error terms in the decomposition from Lemma~\ref{prop:error-decomp-m}.

\begin{proposition}\label{loss-Lip-2}
For any two measurable functions $g_1, g_2:\cX\times\cE\rightarrow\cY$ and any fixed $(x,y)\in\cX\times\cY$, 
\[
|\cL(g_1; x, y) - \cL(g_2; x, y)| \leq 2 \, \mathbb{E}_\varepsilon\cro{\|g_1(x, \varepsilon) - g_2(x, \varepsilon)\|_2}.
\]    
\end{proposition}
The proof is deferred to Section~\ref{proof-sec-decom}. Proposition~\ref{loss-Lip-2} establishes that, for any fixed
\((x,y)\in\mathcal X\times\mathcal Y\), the functional \(g\mapsto \mathcal L(g;x,y)\) is \(2\)-Lipschitz. This property, in particular, implies that the approximation error in the decomposition of Lemma~\ref{prop:error-decomp-m} satisfies
\begin{equation}\label{approx-error}
\mathcal{R}(f)-\mathcal{R}(f^*)\leq 2\bbE_{X,\varepsilon}[\|f(X,\varepsilon)-f^*(X,\varepsilon)\|_2]\leq2\|f-f^*\|_{L^{\infty}(\cX\times\cE)}. \end{equation}

Next, we bound the stochastic error term in Lemma~\ref{prop:error-decomp-m} via the Rademacher complexity of the network class $\cF_{\operatorname{NN}}$. There are several ways to characterize this complexity, for instance, using the Vapnik–Chervonenkis dimension (see, e.g., \cite{peter1,peter2}). Here, we quantify the complexity of $\cF_{\operatorname{NN}}$ via covering numbers, as introduced in Section~\ref{dnns-section}.

\begin{lemma}\label{lem:emp-stoch-error}
Suppose there exists a finite constant $\cB$ such that $\|f\|_{L^\infty(\cX\times\cE)}\leq \cB$ for all $f \in \cF_{\operatorname{NN}} \cup \{f^*\}$. Conditionally on $\cE_n$, we have, almost surely,
\begin{align*}
&\mathbb{E}_{\cD_n}\left[\sup_{f \in \mathcal{F}_{\operatorname{NN}}} \Bigl| \mathcal{R}^{(m)}(f) - {\mathcal{R}}_n^{(m)}(f) \Bigr| \right] \\&\leq\inf_{0<\delta<2\cB}\left(8\delta+\frac{24}{\sqrt{n}}\int_{\delta}^{2\cB}\sqrt{\log\cro{2\coverN\big(\tau/2,\cF_{\operatorname{NN}},\|\cdot\|_{\infty}\big)}}\, d\tau\right).
\end{align*}
\end{lemma}
The proof of Lemma~\ref{lem:emp-stoch-error} is deferred to Section~\ref{proof-sec-decom}. It indicates that the stochastic error is controlled by the metric entropy of the neural network class $\cF_{\operatorname{NN}}$. In particular, for a neural network class $\cF_{\operatorname{NN}}=\cF(L,{\bs p},s,\cB)$ with sufficiently large $L,{\bs p}$ and $s$, choosing $\delta=1/n$ yields $\mathbb{E}_{\cD_n}\left[\sup_{f \in \mathcal{F}_{\operatorname{NN}}} | \mathcal{R}^{(m)}(f) - {\mathcal{R}}_n^{(m)}(f)| \right]\lesssim \cB\sqrt{sL\log(np_{\operatorname{max}}L)/n}$, where $p_{\operatorname{max}}=\max_{i=0,\ldots,L+1}p_i$. Next, we derive an upper bound for the Monte Carlo error term appearing in Lemma~\ref{prop:error-decomp-m}.
\begin{lemma}
\label{lem:expect-mc}
Suppose there exists a finite constant $\cB$ such that $\|f\|_{L^\infty(\cX\times\cE)}\leq \cB$ for all $f \in \cF_{\operatorname{NN}} \cup \{f^*\}$. Then, 
\begin{align*}
&\mathbb{E}_{\cE_n}\cro{\sup_{f\in\cF_{\operatorname{NN}}}\left|\cR(f)-\cR^{(m)}(f)\right|}\\
&\leq14\inf_{0<\delta<2\cB}\left(4\delta+\frac{24}{\sqrt{mn}}\int_{\delta}^{2\cB}\sqrt{\log\cro{2\coverN\big(\tau/2,\cF_{\operatorname{NN}},\|\cdot\|_{\infty}\big)}}\, d\tau\right).
\end{align*}
\end{lemma}
The proof of Lemma~\ref{lem:expect-mc} is provided in Section~\ref{proof-sec-decom}. Essentially, in contrast to Lemma~\ref{lem:emp-stoch-error}, Lemma~\ref{lem:expect-mc} demonstrates that at the population level, the Monte Carlo error does not alter the convergence order in $n$;
for fixed $m$, it is of the same order as the stochastic error up to
a factor of $m^{-1/2}$. However, while the statistical impact is negligible, a small sample of noise in \eqref{loss-2} results in a high-variance empirical loss landscape. This increased variance can lead to jagged gradients, thereby increasing the computational complexity and optimization cost when using gradient-based methods to find a global minimizer. 

Substituting Proposition~\ref{loss-Lip-2}, Lemma~\ref{lem:emp-stoch-error} and Lemma~\ref{lem:expect-mc} into the error decomposition of Lemma~\ref{prop:error-decomp-m} yields the following upper bound for the expected excess risk of $\widehat f_m$.
\begin{theorem}
\label{thm:expect-main-m}
Suppose there exists a finite constant $\cB$ such that $\|f\|_{L^\infty(\cX\times\cE)}\leq \cB$ for all $f \in \cF_{\operatorname{NN}} \cup \{f^*\}$. Then, the expected excess risk of the Monte Carlo estimator $\widehat f_m$ satisfies
\begin{align*}
&\mathbb{E}_{\cD_n,\cE_n}\cro{ \mathcal{R}(\widehat f_m) - \mathcal{R}(f^*)}\\
&\leq 2\inf_{f\in\mathcal{F}_{\mathrm{NN}}}\hspace{-3pt}\|f-f^*\|_{L^\infty(\cX\times\cE)}+16\hspace{-2pt}\inf_{0<\delta_1<2\cB}\hspace{-1pt}\left(\delta_1+\frac{3}{\sqrt{n}}\int_{\delta_1}^{2\cB}\hspace{-10pt}\sqrt{\log\cro{2\coverN\big(\tau/2,\cF_{\operatorname{NN}},\|\cdot\|_{\infty}\big)}}\, d\tau\right)\\
&\quad + 112\inf_{0<\delta_2<2\cB}\left(\delta_2+\frac{6}{\sqrt{mn}}\int_{\delta_2}^{2\cB}\hspace{-10pt}\sqrt{\log\cro{2\coverN\big(\tau/2,\cF_{\operatorname{NN}},\|\cdot\|_{\infty}\big)}}\, d\tau\right).
\end{align*}
\end{theorem}
Substituting the explicit approximation error of $f^*$ and the complexity bounds of $\cF_{\operatorname{NN}}$ yields the convergence rate stated in Theorem~\ref{thm-rate}.

\section{Conclusion}\label{conclusion}
In this paper, we study the statistical error of learning conditional generators using deep sparse ReLU neural networks within the framework of engression. The excess risk is decomposed into bias, Monte Carlo error, and stochastic error components. In particular, under a compositional smoothness assumption on the conditional generator, we establish non-asymptotic convergence rates for the expected excess risk. The resulting rate depends only on the intrinsic compositional structure of the underlying model, rather than the ambient dimension, thereby highlighting the benefit of exploiting hierarchical low-dimensional structure in high-dimensional learning problems. Moreover, the generative modeling framework generalizes the classical regression setting, demonstrating the broader applicability of deep sparse neural networks to more flexible modeling scenarios.

The present analysis relies on the assumption that the noise variables have compact support, for instance on $[0,1]^k$. In more general settings, for example when the noise follows a multivariate Gaussian distribution, we conjecture that an extension of the present analysis, together with the truncation approach introduced in \cite{Huang2022GANerror}, may yield a characterization of the unbounded case.

\section{Proofs and additional results}\label{proofs-sec}
\subsection{Auxiliary results}
We first establish several preliminary results required to prove the main theorem.
\begin{proposition}\label{bound-value-cL}
Suppose that there exists a finite constant $\cB$ such that $\|f\|_{\infty}\leq \cB$ for all $f \in \cF_{\operatorname{NN}}$. Then, for any $(x,y)\in\cX\times\cY$ satisfying $\|y\|_{\infty}\leq \cB$ and any $f \in \cF_{\operatorname{NN}}$, it holds that 
\begin{equation}\label{cl-bound}
\big|\cL(f; x, y)\big| \leq 2\cB,
\end{equation}
and for any ${\bs{\varepsilon}}=(\varepsilon_1,\ldots,\varepsilon_m)^{\top}\in\cE^m$, \begin{equation}\label{clm-bound}
\big|\widehat\cL_m\big(f;x,y,{\bs{\varepsilon}}\big)\big|\leq2\cB,\quad\mbox{almost surely}.
\end{equation}
\end{proposition}

\begin{proof}
We prove \eqref{cl-bound} only; the proof of \eqref{clm-bound} follows by a similar argument. By the triangle inequality and the assumption that $\|f\|_{\infty}\leq \cB$ so that $\|y\|_2\leq \cB$, we have
\begin{equation*}
\|y - f(x, \varepsilon)\|_2 \leq \|y\|_2 + \|f(x, \varepsilon)\|_2 \leq 2\cB.
\end{equation*}
Similarly, it follows that $\|f(x, \varepsilon) - f(x, \varepsilon')\|_2 \leq \|f(x, \varepsilon)\|_2 + \|f(x, \varepsilon')\|_2 \leq2\cB$. Recalling the definition of the energy distance loss in \eqref{L-exp-epsilon}, we observe $-\cB\leq\cL(f; x, y) \leq  2\cB.$ The result follows immediately.
\end{proof}

Analogous to Proposition~\ref{loss-Lip-2} in the main text, the $m$-sample Monte Carlo loss $\widehat\cL_m$ also exhibits a Lipschitz-type property.
\begin{proposition}\label{fm-loss-2}
For any functions $g_1, g_2:\cX\times\cE\rightarrow\cY$, fixed $(x,y)\in\cX\times\cY$, and noise vector ${\bs{\varepsilon}}=(\varepsilon_1,\ldots,\varepsilon_m)^{\top}\in\cE^{m}$, it holds that
\[
\left|\widehat\cL_m(g_1;x,y,{\bs{\varepsilon}}) - \widehat\cL_m(g_2; x, y,{\bs{\varepsilon}})\right| \leq 2 \|g_1 - g_2\|_{L^\infty(\cX\times\cE)}.
\]        
\end{proposition}
\begin{proof}
Applying the definition of the $m$-sample loss $\widehat{\mathcal{L}}_m$ from \eqref{fm-loss-def} and the triangle inequality, we have
\begin{align*}
&\left|\widehat\cL_m(g_1;x,y,{\bs{\varepsilon}}) - \widehat\cL_m(g_2; x, y,{\bs{\varepsilon}})\right|\\ &\leq\left|\frac{1}{m}\sum_{j=1}^m\cro{\|y-g_1(x,\varepsilon_j)\|_2-\|y-g_2(x,\varepsilon_j)\|_2}\right|\\
&\quad+\left|\frac{1}{2m(m-1)}\sum_{j\not=j'}^m\left(\|g_1(x,\varepsilon_j)-g_1(x,\varepsilon_{j'})\|_2-\|g_2(x,\varepsilon_j)-g_2(x,\varepsilon_{j'})\|_2\right)\right|.
\end{align*}
Using the triangle inequality, $\big|\|a-b\|_2-\|c-d\|_2\big|\leq\|a-c\|_2+\|b-d\|_2$, we obtain \begin{align*} &\left| \widehat{\cL}_m(g_1;x,y,\bs{\varepsilon}) - \widehat{\cL}_m(g_2;x,y,\bs{\varepsilon}) \right|\\ &\leq \frac{1}{m}\sum_{j=1}^m \|g_1(x,\varepsilon_j)-g_2(x,\varepsilon_j)\|_2\\ &\quad+ \frac{1}{2m(m-1)} \sum_{\substack{1\leq j,j'\leq m\\ j\neq j'}} \Bigl( \|g_1(x,\varepsilon_j)-g_2(x,\varepsilon_j)\|_2 + \|g_1(x,\varepsilon_{j'})-g_2(x,\varepsilon_{j'})\|_2 \Bigr). \end{align*} Each of the two sums is bounded by $\|g_1-g_2\|_{L^\infty(\cX\times\cE)}$. Hence, \[ \left| \widehat{\cL}_m(g_1;x,y,\bs{\varepsilon}) - \widehat{\cL}_m(g_2;x,y,\bs{\varepsilon}) \right| \leq 2\|g_1-g_2\|_{L^\infty(\cX\times\cE)}, \] which completes the proof.
\end{proof}
To bound the stochastic error, we repeatedly use the refined Dudley inequality stated as below. It follows by a straightforward modification of the proof of Lemma~27.4 in \citet{shalevshwartz2014understanding}.
\begin{lemma}\label{dudly-int}
Let $\cA\subseteq\R^n$, and suppose that there exist
$\bs a_0\in\R^n$ and $K>0$ such that
\[
\sup_{\bs a\in\cA}
\|\bs a-\bs a_0\|_2
\leq K.
\]
Then, for any integer $M>0$,
\[
\frac{1}{n}
\E_{\bs\sigma}
\left[
\sup_{\bs a\in\cA}
\sum_{i=1}^n\sigma_i a_i
\right]
\leq
K2^{-M}
+
\frac{6K}{\sqrt n}
\sum_{m=1}^M
2^{-m}
\sqrt{
\log\cN
\left(
K2^{-m},
\cA,
\|\cdot\|_2
\right)
},
\]
where ${\bs\sigma}=(\sigma_1,\ldots,\sigma_n)^{\top}$ is a vector of i.i.d. Rademacher random variables.
\end{lemma}

We also use the following inequality for decoupling $U$-statistics, obtained by taking $\Phi$ in Theorem~3.1.1 of \cite{delapena} as the identity function.
\begin{lemma}\label{delapena-decop}
For natural numbers $n \geq m$, let $\{X_i\}_{i=1}^n$ be $n$ independent random variables with values in a measurable space $(S,\mathcal{S})$, and let $\{X_i^{(k)}\}_{i=1}^n$, $k = 1,\ldots,m$, be $m$ independent copies of this sequence. 
Let $B$ be a separable Banach space and, for each $(i_1,\ldots,i_m) \in \mathcal{I}_m^n$, let $h_{i_1\ldots i_m}: S^m \to B$ be measurable functions such that $$\mathbb{E}(\|h_{i_1\ldots i_m}(X_{i_1},\ldots,X_{i_m})\|) < \infty.$$ Then,
\[
\mathbb{E}\cro{\left\| \sum_{\mathcal{I}_m^n} h_{i_1\ldots i_m}(X_{i_1},\ldots,X_{i_m}) \right\|}\leq \mathbb{E}\cro{C_m \left\| \sum_{\mathcal{I}_m^n} h_{i_1\ldots i_m}(X_{i_1}^{(1)},\ldots,X_{i_m}^{(m)}) \right\|},
\]
where $C_m = 2^m (m^m - 1)\cro{(m-1)^{(m-1)} - 1}\times \cdots \times 3$.
\end{lemma}

\subsection{Proofs in Section~\ref{rate}}\label{proof-sec-decom}
\begin{proof}[Proof of Proposition~\ref{loss-Lip-2}]
Define $A_1(f) = \mathbb{E}_{\varepsilon}[\|y - f(x, \varepsilon)\|_2]$, $A_2(f) = \mathbb{E}_{\varepsilon, \varepsilon'}[\|f(x, \varepsilon) - f(x, \varepsilon')\|_2]/2$,
so that, by the definition in \eqref{L-exp-epsilon}, $\mathcal{L}(f; x, y) = A_1(f) - A_2(f)$.
For any two measurable functions $g_1, g_2:\cX\times\cE\rightarrow\cY$, using the triangle inequality, we have
\begin{align}
\left|\mathcal{L}(g_1; x, y) - \mathcal{L}(g_2; x, y)\right| 
\leq\left|A_1(g_1) - A_1(g_2)\right| + \left|A_2(g_1) - A_2(g_2)\right|.\label{sep-1-2}
\end{align}

We now bound each term separately on the right-hand side of \eqref{sep-1-2}. For the first term,
\begin{align*}
|A_1(g_1) - A_1(g_2)| 
&= \Bigl| \mathbb{E}_{\varepsilon}[\|y - g_1(x, \varepsilon)\|_2] - \mathbb{E}_{\varepsilon}[\|y - g_2(x, \varepsilon)\|_2] \Bigr| \\
&\leq \mathbb{E}_{\varepsilon}\bigl[ \bigl| \|y - g_1(x, \varepsilon)\|_2 - \|y - g_2(x, \varepsilon)\|_2 \bigr| \bigr]\\
&\leq \mathbb{E}_{\varepsilon}\bigl[ \|g_1(x, \varepsilon) - g_2(x, \varepsilon)\|_2 \bigr],
\end{align*}
where the first inequality follows from Jensen’s inequality, and the second follows from the triangle inequality.

For the second term in \eqref{sep-1-2}, observe that
\begin{equation}\label{tool-inequa}
\big|
\|a-b\|_2-\|c-d\|_2
\big|
\leq
\|a-c\|_2+\|b-d\|_2.
\end{equation}
Then, applying \eqref{tool-inequa} and  Jensen’s inequality yields
\begin{align*}
\left|A_2(g_1) - A_2(g_2)\right| &\leq \frac{1}{2} \mathbb{E}_{\varepsilon, \varepsilon'}\bigl[ \bigl| \|g_1(x, \varepsilon) - g_1(x, \varepsilon')\|_2 - \|g_2(x, \varepsilon) - g_2(x, \varepsilon')\|_2 \bigr| \bigr]\\
&\leq \frac{1}{2} \mathbb{E}_{\varepsilon, \varepsilon'}\bigl[ \|g_1(x, \varepsilon) - g_2(x, \varepsilon)\|_2 + \|g_1(x, \varepsilon') - g_2(x, \varepsilon')\|_2 \bigr] \\
&= \frac{1}{2} \Bigl( \mathbb{E}_{\varepsilon}[\|g_1(x, \varepsilon) - g_2(x, \varepsilon)\|_2] + \mathbb{E}_{\varepsilon'}[\|g_1(x, \varepsilon') - g_2(x, \varepsilon')\|_2] \Bigr) \\
&= \mathbb{E}_{\varepsilon}[\|g_1(x, \varepsilon) - g_2(x, \varepsilon)\|_2].
\end{align*}

Putting the two parts together, we obtain
\begin{align*}
|\mathcal{L}(g_1; x, y) - \mathcal{L}(g_2; x, y)|&\leq |A_1(g_1) - A_1(g_2)| + |A_2(g_1) - A_2(g_2)| \\
&\leq \mathbb{E}_{\varepsilon}[\|g_1(x, \varepsilon) - g_2(x, \varepsilon)\|_2] + \mathbb{E}_{\varepsilon}[\|g_1(x, \varepsilon) - g_2(x, \varepsilon)\|_2] \\
&= 2 \, \mathbb{E}_{\varepsilon}[\|g_1(x, \varepsilon) - g_2(x, \varepsilon)\|_2],
\end{align*}
which completes the proof.
\end{proof}

\begin{proof}[Proof of Lemma~\ref{lem:emp-stoch-error}]
Let $\mathcal{D}'_n = \{(X_i', Y_i')\}_{i=1}^n$ denote an independent copy of $\mathcal{D}_n$. In the following analysis, we initially treat the noise samples $\mathcal{E}_n = (\bm{\varepsilon}^{(1)}, \ldots, \bm{\varepsilon}^{(n)})$ as fixed. Recall that, by definition \eqref{def-cRm}, $$\mathcal{R}^{(m)}(f) = \frac{1}{n} \sum_{i=1}^n \mathbb{E}_{\cD_n} \left[ \widehat{\mathcal{L}}_m(f; Z_i, \bm{\varepsilon}^{(i)}) \right].$$
It follows from a symmetrization argument applied to the independent (conditional on $\mathcal{E}_n$) terms that
\begin{align}
&\mathbb{E}_{\cD_n}\left[\sup_{f \in \mathcal{F}_{\operatorname{NN}}} \bigl| \mathcal{R}^{(m)}(f) - {\mathcal{R}}_n^{(m)}(f) \bigr| \right]\nonumber\\
&=\mathbb{E}_{\cD_n} \left[ \sup_{f \in \mathcal{F}_{\operatorname{NN}}}  \left| \mathbb{E}_{\cD'_n} \cro{ \frac{1}{n}\sum_{i=1}^n\widehat\cL_m\big(f; X_i', Y_i', \bm{\varepsilon}^{(i)}\big)} - \frac{1}{n}\sum_{i=1}^n\widehat\cL_m\big(f; X_i, Y_i, \bm{\varepsilon}^{(i)}\big) \right| \right]\nonumber\\
&\leq \mathbb{E}_{\cD_n} \left[\sup_{f \in \mathcal{F}_{\operatorname{NN}}} \mathbb{E}_{\cD'_n} \left| \frac{1}{n}\sum_{i=1}^n \cro{ \widehat\cL_m\big(f; X_i', Y_i', \bm{\varepsilon}^{(i)}\big) - \widehat\cL_m\big(f; X_i, Y_i, \bm{\varepsilon}^{(i)}\big)} \right| \right]\nonumber\\
&\leq\mathbb{E}_{\cD_n,\cD'_n} \left[ \sup_{f \in \mathcal{F}_{\operatorname{NN}}} \Bigl| \frac{1}{n}\sum_{i=1}^n \bigl(\widehat\cL_m(f; X_i', Y_i',{\bs{\varepsilon}}^{(i)})-\widehat\cL_m(f; X_i, Y_i,{\bs{\varepsilon}}^{(i)}) \bigr) \Bigr| \right].\label{sym-bound-fm}  
\end{align}

Define $\boldsymbol{\sigma} = (\sigma_1, \dots, \sigma_n)^{\top}$, where $\sigma_1,\ldots,\sigma_n$ are i.i.d. Rademacher random variables. These variables are independent of the data $\mathcal{D}_n$ and $\mathcal{D}'_n$. Swapping $(X_i,Y_i)$ and $(X_i',Y_i')$ does not change the distribution, hence we can further bound \eqref{sym-bound-fm} as
\begin{align}
&\mathbb{E}_{\cD_n,\cD'_n} \left[ \sup_{f \in \mathcal{F}_{\operatorname{NN}}} \Bigl| \frac{1}{n}\sum_{i=1}^n \bigl(\widehat\cL_m(f; X_i', Y_i',{\bs{\varepsilon}}^{(i)})-\widehat\cL_m(f; X_i, Y_i,{\bs{\varepsilon}}^{(i)}) \bigr) \Bigr| \right]\nonumber\\
&\leq 2 \, \mathbb{E}_{\cD_n,{\bs{\sigma}}}\!\left[ \sup_{f \in \mathcal{F}_{\operatorname{NN}}}\Bigl| \frac{1}{n}\sum_{i=1}^n \sigma_i\widehat\cL_m\big(f; X_i, Y_i,{\bs{\varepsilon}}^{(i)}\big) \Bigr| \right].\label{sym-b-fm} 
\end{align}

For any set $\cM \subseteq \mathbb{R}^n$, the Rademacher complexity of $\cM$ is defined as
\begin{equation}\label{comp-def}
\mathfrak{R}(\cM) = \mathbb{E}_{{\bs{\sigma}}}\left[\sup_{(\zeta_1,\ldots,\zeta_n) \in \cM} \frac{1}{n}\sum_{i=1}^n \sigma_i \zeta_i \right].
\end{equation}
Conditionally on the observed data $\cD_n = \{Z_1, \dots, Z_n\}$ and the noise samples $\cE_n$, we denote the set of loss vectors by
\begin{equation*}
\cG_{|\cD_n,\cE_n} = \left\{ \big(\widehat\cL_m(f;Z_1,{\bs{\varepsilon}^{(1)}}), \ldots, \widehat\cL_m(f; Z_n,{\bs{\varepsilon}^{(n)}})\big) : f \in \cF_{\operatorname{NN}} \right\}.
\end{equation*}
Define the symmetric completion $\check{\cG}_{|\cD_n,\cE_n}$ of $\cG_{|\cD_n,\cE_n}$ as $\check{\cG}_{|\cD_n,\cE_n} = \cG_{|\cD_n,\cE_n} \cup \{-\bs{\zeta} : \bs{\zeta} \in \cG_{|\cD_n,\cE_n}\}.$ It follows from \eqref{sym-b-fm} that
\begin{equation}\label{rad-comp-final-fm}
\mathbb{E}_{\cD_n}\left[\sup_{f \in \mathcal{F}_{\operatorname{NN}}} \bigl|\mathcal{R}^{(m)}(f) - {\mathcal{R}}_n^{(m)}(f) \bigr| \right] \leq 2 \, \mathbb{E}_{\cD_n} \big[ \mathfrak{R}(\check{\cG}_{|\cD_n,\cE_n}) \big].
\end{equation}

Under the conditions that $\|f\|_{\infty}\leq \cB$ for all $f\in\cF_{\operatorname{NN}}$ and $\|f^*\|_{\infty}\leq \cB$ (implying $\|y\|_{\infty}\leq \cB$), we apply Proposition~\ref{bound-value-cL} to obtain that for any given point $(x,y)$,
\begin{equation}\label{bound-for-cL-fm}
\big|\widehat\cL_m(f;x,y,{\bs\varepsilon}^{(i)})\big|\leq2\cB.   
\end{equation}

Recall that for any given two vectors ${\bs u}, {\bs v}\in\R^n$, 
$$\|{\bs u}-{\bs v}\|_2=\sqrt{\frac{1}{n}\sum_{i=1}^n(u_i-v_i)^2}.$$ Denote the $\delta$-covering number of the set $\check{\cG}_{|\cD_n,\cE_n}\subseteq\R^n$ with respect to the normalized Euclidean distance $\|\cdot\|_2$ by $\coverN\big(\delta,\check{\cG}_{|\cD_n,\cE_n},\|\cdot\|_2\big)$. Since any $\delta$-cover of $\cG_{|\cD_n,\cE_n}$ can be mirrored to form a $\delta$-cover of $-\cG_{|\cD_n,\cE_n}$, the covering numbers satisfy $$\mathcal{N}\big(\delta,\check{\cG}_{|\cD_n,\cE_n}, \|\cdot\|_2\big) \leq 2\mathcal{N}\big(\delta, \cG_{|\cD_n,\cE_n}, \|\cdot\|_2\big).$$ Observe that from \eqref{bound-for-cL-fm}, we have for any ${\bs{\zeta}},{\bs{\zeta}}'\in\check{\cG}_{|\cD_n,\cE_n}\subseteq\R^n$, $\|{\bs{\zeta}}-{\bs{\zeta}}'\|_2\leq4\cB$. Therefore, applying Lemma~\ref{dudly-int} to the set $\check{\cG}_{|\cD_n,\cE_n}$ with $K=4\cB$, we obtain that for any nonnegative integer $M$,
\begin{align}
\mathfrak{R}(\check{\cG}_{|\cD_n,\cE_n})&\leq 2^{-M+2}\cB+\frac{24\cB}{\sqrt{n}}\sum_{m=1}^{M}2^{-m}\sqrt{\log\cN\big(2^{-m+2}\cB,\check{\cG}_{|\cD_n,\cE_n},\|\cdot\|_2\big)}\nonumber\\
&\leq 2^{-M+2}\cB + \frac{12}{\sqrt{n}} \int_{2^{-M+1}\cB}^{2\cB} \sqrt{\log\cro{2\mathcal{N} \big( \tau,\cG_{|\cD_n,\cE_n}, \|\cdot\|_2 \big)}} \, d\tau.\label{sum-integ}
\end{align}
Observe that for any $\delta\in(0,2\cB)$, there exists an integer $M$ satisfying $2^{-M}\cB\leq\delta<2^{-M+1}\cB$. It then follows that
\begin{equation}\label{bb-1-fm}
\mathfrak{R}(\check{\cG}_{|\cD_n,\cE_n})\leq\inf_{0<\delta<2\cB}\left(4\delta+\frac{12}{\sqrt{n}}\int_{\delta}^{2\cB}\sqrt{\log\cro{2\coverN\big(\tau,\cG_{|\cD_n,\cE_n},\|\cdot\|_2\big)}}\, d\tau\right).
\end{equation}
Moreover, for any two functions $f,g\in\cF_{\operatorname{NN}}$ satisfying $\|f-g\|_{L^{\infty}(\cX\times\cE)}\leq\delta/2$, Proposition~\ref{fm-loss-2} implies that almost surely
\begin{align*}
\left|\widehat\cL_m\big(f;X_i,Y_i,{\bs{\varepsilon}}^{(i)}\big)-\widehat\cL_m\big(g;X_i,Y_i,{\bs{\varepsilon}}^{(i)}\big)\right|\leq2\|f-g\|_{L^{\infty}(\cX\times\cE)}\leq\delta.
\end{align*}
Thus, we have
\begin{equation}\label{l2-sup-cover-conne-fm}
\log\coverN(\delta,\cG_{|\cD_n,\cE_n},\|\cdot\|_2)\leq  \log\coverN(\delta/2,\cF_{\operatorname{NN}},\|\cdot\|_{\infty}).
\end{equation}
Substituting the covering number bound from \eqref{l2-sup-cover-conne-fm} into inequality \eqref{bb-1-fm} implies, almost surely, that
\begin{equation}\label{rader-bb-fm}
\mathfrak{R}(\check{\cG}_{|\cD_n,\cE_n})\leq\inf_{0<\delta<2\cB}\left(4\delta+\frac{12}{\sqrt{n}}\int_{\delta}^{2\cB}\sqrt{\log\cro{2\coverN\big(\tau/2,\cF_{\operatorname{NN}},\|\cdot\|_{\infty}\big)}}\, d\tau\right).    
\end{equation}
Taking the expectation with respect to the sample $\cD_n$ on both sides of \eqref{rader-bb-fm}, and combining this with \eqref{rad-comp-final-fm}, yields
\begin{align*}
&\mathbb{E}_{\cD_n}\left[\sup_{f \in \mathcal{F}_{\operatorname{NN}}}\Bigl| \mathcal{R}^{(m)}(f) - {\mathcal{R}}_n^{(m)}(f) \Bigr| \right]\\
&\leq\inf_{0<\delta<2\cB}\left(8\delta+\frac{24}{\sqrt{n}}\int_{\delta}^{2\cB}\sqrt{\log\cro{2\coverN\big(\tau/2,\cF_{\operatorname{NN}},\|\cdot\|_{\infty}\big)}}\, d\tau\right),    
\end{align*}
which completes the proof.
\end{proof}

\begin{proof}[Proof of Lemma~\ref{lem:expect-mc}]
Recall that for any $Z = (X,Y)\in\cX\times\cY$ and any given $\bm{\varepsilon} = (\varepsilon_1, \ldots, \varepsilon_m)\in\cE^m$, the loss function $\widehat{\mathcal{L}}_m$ defined in \eqref{fm-loss-def} takes the form
\[
\widehat{\mathcal{L}}_m(f; Z, \bm{\varepsilon}) 
= \frac{1}{m}\sum_{j=1}^{m}\|Y - f(X,\varepsilon_j)\|_2 
- \frac{1}{2m(m-1)}\sum_{j\not=j'}^{m}\|f(X,\varepsilon_j) - f(X,\varepsilon_{j'})\|_2,
\]
and its expectation with respect to the noise samples $\bm{\varepsilon} = (\varepsilon_1, \ldots, \varepsilon_m)$ satisfies $$\mathcal{L}(f; Z) = \mathbb{E}_{\bm{\varepsilon}}[\widehat{\mathcal{L}}_m(f; Z, \bm{\varepsilon})].$$ 

Let $\widetilde\cE_n=(\widetilde{\bs\varepsilon}^{(1)},\ldots,\widetilde{\bs\varepsilon}^{(n)})$ be an independent copy of $\cE_n=({\bs\varepsilon}^{(1)},\ldots,{\bs\varepsilon}^{(n)})$ and let ${\bs\sigma}=(\sigma_1,\ldots,\sigma_n)^{\top}$ be i.i.d. Rademacher random variables. By the definition of $\cR^{(m)}(f)$ in \eqref{def-cRm}, and using $$\cR(f)=\E_{\cD_n}\cro{\frac{1}{n}\sum_{i=1}^n\cL(f;Z_i)},$$ it follows that
\begin{align} 
&\mathbb{E}_{\cE_n} \cro{\sup_{f\in\cF_{\operatorname{NN}}}\left|\cR(f)-\cR^{(m)}(f)\right|} \nonumber\\
&= \mathbb{E}_{\cE_n} \cro{\sup_{f\in\cF_{\operatorname{NN}}}\left|\mathbb{E}_{\cD_n} \cro{\frac{1}{n}\sum_{i=1}^n \left( \mathcal{L}(f; Z_i) - \widehat{\mathcal{L}}_m(f; Z_i, \bm{\varepsilon}^{(i)}) \right)} \right|} \nonumber\\
&\leq \mathbb{E}_{\cD_n,\cE_n}\cro{\sup_{f\in\cF_{\operatorname{NN}}}\left|\frac{1}{n}\sum_{i=1}^n \left( \mathcal{L}(f; Z_i) - \widehat{\mathcal{L}}_m(f; Z_i, \bm{\varepsilon}^{(i)}) \right) \right|}.\label{rewrite-monte-carlo}
\end{align}

We now treat $Z_i = (X_i,Y_i)$ as fixed and consider the deviation arising solely from the randomness of the noise samples $\cE_n = (\bm{\varepsilon}^{(1)}, \ldots, \bm{\varepsilon}^{(n)})$. Define
\[ 
\phi_f(x,y,\varepsilon) = \|y - f(x, \varepsilon)\|_2 \quad \text{and} \quad \psi_f(x,\varepsilon, \varepsilon') = \|f(x, \varepsilon) - f(x, \varepsilon')\|_2. 
\]
Using the triangle inequality, we split the linear empirical process and the U-statistic component in \eqref{rewrite-monte-carlo} to obtain the upper bound
\begin{align}
&\mathbb{E}_{\cD_n,\cE_n}\cro{\sup_{f\in\cF_{\operatorname{NN}}}\left|\frac{1}{n}\sum_{i=1}^n \cro{ \mathcal{L}(f; Z_i) - \widehat{\mathcal{L}}_m(f; Z_i, \bm{\varepsilon}^{(i)})}\right|}\label{total-dec-monte}\\
&\leq \underbrace{\mathbb{E}_{\cD_n,\cE_n} \left[ \sup_{f \in \mathcal{F}_{\operatorname{NN}}} \left|\frac{1}{n}\sum_{i=1}^n\cro{ \mathbb{E}_{\varepsilon}[\phi_f(X_i,Y_i,\varepsilon)] - \frac{1}{m}\sum_{j=1}^m \phi_f(X_i,Y_i,{\bs\varepsilon}_j^{(i)})}\right| \right]}_{I_1}\nonumber\\
&\quad+ \frac{1}{2} \underbrace{\mathbb{E}_{\cD_n,\cE_n} \left[ \sup_{f \in \mathcal{F}_{\operatorname{NN}}} \left| \frac{1}{n}\sum_{i=1}^n\cro{\mathbb{E}_{\varepsilon, \varepsilon'}[\psi_f(X_i,\varepsilon, \varepsilon')] - \frac{1}{m(m-1)}\sum_{j \neq k} \psi_f(X_i,{\bs\varepsilon}_j^{(i)}, {\bs\varepsilon}_k^{(i)})} \right| \right]}_{I_2}.\nonumber    
\end{align}
\paragraph{Bounding $I_1$:}
First, fix $\cD_n=\{(X_i,Y_i)\}_{i=1}^n$ and define
$$\zeta_f(x,y,\varepsilon)=\mathbb{E}_{\varepsilon'}[\phi_f(x,y,\varepsilon')] - \phi_f(x,y,\varepsilon).$$ Under the condition that for all $f\in\cF_{\operatorname{NN}}\cup\{f^*\}$, $\|f\|_{L^{\infty}(\cX\times\cE)}\leq\cB$, for any $(x,y)\in\mathcal{X}\times\mathcal{Y}$ generated by model \eqref{eq:statistical-setting}, and any $\varepsilon\in\mathcal{E}$, it holds that
\begin{equation}\label{eta-abso-b}
\left|\phi_f(x,y,\varepsilon)\right|\leq2\cB.    
\end{equation}
Moreover, for any $g\in\cF_{\operatorname{NN}}$, it can be shown, by an argument analogous to that in the proof of Proposition~\ref{loss-Lip-2}, that
\begin{align}
\left|\phi_f(x,y,\varepsilon)-\phi_g(x,y,\varepsilon)\right|\leq\|f-g\|_{L^\infty(\cX\times\cE)}.\label{cover-ralate-I1}
\end{align}

Reformulating $I_1$ as
\begin{align}
I_1=\mathbb{E}_{\cD_n}\cro{\mathbb{E}_{\cE_n} \left[ \sup_{f \in \mathcal{F}_{\operatorname{NN}}} \left|\frac{1}{n}\sum_{i=1}^n\frac{1}{m}\sum_{j=1}^m\zeta_f(X_i,Y_i,{\bs\varepsilon}_j^{(i)})\right|\Big|\cD_n\right]}.
\end{align}
Conditioned on $\cD_n$, the random variables $\{\zeta_f\big(X_i,Y_i,\bm{\varepsilon}_j^{(i)}\big)\}_{i,j}$ are independent. A standard symmetrization argument, as in the proof of Lemma~\ref{lem:emp-stoch-error}, thus yields
\begin{align}
I_1 \leq 2 \, \mathbb{E}_{\cD_n,\cE_n,\bm{\sigma}} \left[ \sup_{f \in \mathcal{F}_{\operatorname{NN}}}\left| \frac{1}{mn}\sum_{i=1}^n\sum_{j=1}^m \sigma_j^{(i)} \phi_f(X_i,Y_i,\bm{\varepsilon}_j^{(i)})\right| \right],    
\end{align}
where $\bm{\sigma} = (\bm{\sigma}_{1}, \ldots, \bm{\sigma}_{m})^{\top}$, and each $\bm{\sigma}_{j}=(\sigma^{(1)}_j,\ldots,\sigma^{(n)}_j)^{\top}$ is a vector of $n$ i.i.d. Rademacher random variables. 

Now consider the set of vectors $$\mathcal{V}_{|\cD_n,\cE_n} =\left\{\big(\phi_f(Z_1,\bm{\varepsilon}_1^{(1)}),\ldots,\phi_f(Z_1,\bm{\varepsilon}_m^{(1)}),\ldots,\phi_f(Z_n,\bm{\varepsilon}_1^{(n)}),\ldots,\phi_f(Z_n,\bm{\varepsilon}_m^{(n)})\big): f \in \mathcal{F}_{\operatorname{NN}}\right\}$$ and define its symmetric completion as $${\mathcal{\check{V}}_{|\cD_n,\cE_n}}=\mathcal{V}_{|\cD_n,\cE_n}\cup\{-{\bs\eta}:{\bs\eta}\in\mathcal{V}_{|\cD_n,\cE_n}\}.$$ It follows that for any $\delta>0$,
\begin{equation}\label{cover-bound-monte}
\coverN(\delta, {\mathcal{\check{V}}_{|\cD_n,\cE_n}}, \|\cdot\|_2) \leq2\coverN(\delta, \mathcal{V}_{|\cD_n,\cE_n}, \|\cdot\|_2). 
\end{equation}
Therefore, using the notation introduced in \eqref{comp-def} and ${\bs\eta}= (\bm{\eta}_{1}, \ldots, \bm{\eta}_{m})^{\top}$, and each $\bm{\eta}_{j}=(\eta^{(1)}_j,\ldots,\eta^{(n)}_j)^{\top}$, we have
\begin{equation}\label{connect-I1-rader}
I_1\leq2  \E_{\cD_n,\cE_n}\cro{\mathbb{E}_{\bm{\sigma}} \left[ \sup_{{\bs\eta}\in {\mathcal{\check{V}}_{|\cD_n,\cE_n}}}\frac{1}{mn}\sum_{i=1}^n\sum_{j=1}^m \sigma_j^{(i)} \eta_j^{(i)}\right]}=2\E_{\cD_n,\cE_n}\cro{\mathfrak{R}(\check{\cV}_{|\cD_n,\cE_n})}.  
\end{equation}

For any ${\bs\eta},{\bs\eta}'\in\check{\cV}_{|\cD_n,\cE_n}\subseteq\R^{mn}$, it follows from \eqref{eta-abso-b} that $\|{\bs\eta}-{\bs\eta}'\|_2\leq4\cB$. Applying Lemma~\ref{dudly-int} to the set ${\mathcal{\check{V}}_{|\cD_n,\cE_n}}$ with $K=4\cB$, we can show that for any nonnegative integer $M$,
\begin{align*}
\mathfrak{R}(\check{\cV}_{|\cD_n,\cE_n})&\leq 2^{-M+2}\cB+\frac{24\cB}{\sqrt{mn}}\sum_{k=1}^{M}2^{-k}\sqrt{\log\cN\big(2^{-k+2}\cB,\check{\cV}_{|\cD_n,\cE_n},\|\cdot\|_2\big)}\\
&\leq 2^{-M+2}\cB + \frac{24}{\sqrt{mn}} \int_{2^{-M+1}\cB}^{2\cB} \sqrt{\log \mathcal{N} \big( \tau, \check{\cV}_{|\cD_n,\cE_n}, \|\cdot\|_2 \big)} \, d\tau\\
&\leq 2^{-M+2}\cB + \frac{24}{\sqrt{mn}} \int_{2^{-M+1}\cB}^{2\cB} \sqrt{\log\cro{2\mathcal{N} \big( \tau,{\cV}_{|\cD_n,\cE_n}, \|\cdot\|_2 \big)}} \, d\tau.
\end{align*} 
Since for any $\delta\in(0,2\cB)$, there exists an integer $M$ satisfying $2^{-M}\cB\leq\delta<2^{-M+1}\cB$. It then follows that
\begin{equation}\label{monte-carlo-initial-bound}
\mathfrak{R}(\check{\cV}_{|\cD_n,\cE_n})\leq\inf_{0<\delta<2\cB}\left(4\delta+\frac{24}{\sqrt{mn}}\int_{\delta}^{2\cB}\sqrt{\log\cro{2\mathcal{N} \big( \tau,{\cV}_{|\cD_n,\cE_n}, \|\cdot\|_2 \big)}}\, d\tau\right).    
\end{equation}
Furthermore, for any two functions $f,g\in\cF_{\operatorname{NN}}$ satisfying $\|f-g\|_{L^{\infty}(\cX\times\cE)}\leq\delta$, \eqref{cover-ralate-I1} implies that almost surely
$$\left|\phi_f(Z_i,{\bs{\varepsilon}}_j^{(i)})-\phi_g(Z_i,{\bs{\varepsilon}}_j^{(i)})\right|\leq\delta.$$
This, together with \eqref{monte-carlo-initial-bound}, yields
\begin{equation*}
\mathfrak{R}(\check{\cV}_{|\cD_n,\cE_n})\leq\inf_{0<\delta<2\cB}\left(4\delta+\frac{24}{\sqrt{mn}}\int_{\delta}^{2\cB}\sqrt{\log\cro{2\coverN\big(\tau,\cF_{\operatorname{NN}},\|\cdot\|_{\infty}\big)}}\, d\tau\right),
\end{equation*}
Combining the above with \eqref{connect-I1-rader} shows that
\begin{equation}\label{bound-monte-I1}
I_1 \leq \inf_{0<\delta<2\cB}\left(8\delta+\frac{48}{\sqrt{mn}}\int_{\delta}^{2\cB}\sqrt{\log\cro{2\coverN\big(\tau,\cF_{\operatorname{NN}},\|\cdot\|_{\infty}\big)}}\, d\tau\right).    
\end{equation}

\paragraph{Bounding $I_2$:}
Define the centered kernel
\[
\kappa_f(x,\varepsilon, \varepsilon') = \psi_f(x,\varepsilon, \varepsilon') - \mathbb{E}_{\tilde\varepsilon,\tilde\varepsilon'}[\psi_f(x,\tilde\varepsilon, \tilde\varepsilon')].
\]
By the triangle inequality, we have
\begin{align}
\left|\psi_f(x,\varepsilon, \varepsilon')-\psi_g(x,\varepsilon, \varepsilon')\right|&=\left|\|f(x,\varepsilon)-f(x,\varepsilon')\|_2-\|g(x,\varepsilon)-g(x,\varepsilon')\|_2\right|\nonumber\\
&\leq\|f(x,\varepsilon)-g(x,\varepsilon)\|_2+\|g(x,\varepsilon')-f(x,\varepsilon')\|_2\nonumber\\
&\leq2\|f-g\|_{L^\infty(\cX\times\cE)}.\label{psi-lip-I2}
\end{align}
Now rewrite $I_2$ as
\begin{equation}\label{I2-sum-U}
I_2 = \mathbb{E}_{\cD_n}\cro{\mathbb{E}_{\cE_n} \left[ \sup_{f \in \mathcal{F}_{\operatorname{NN}}} \left|\frac{1}{n}\sum_{i=1}^n\frac{1}{m(m-1)}\sum_{j \neq k} \kappa_f(X_i,{\bs\varepsilon}_j^{(i)}, {\bs\varepsilon}_k^{(i)}) \right| \Big|\cD_n\right]}.  \end{equation}
Since the summands in a U-statistic of order 2 are dependent, we cannot directly apply symmetrization to \eqref{I2-sum-U} as written. Let $\widetilde\cE_n = (\tilde{\bs\varepsilon}^{(1)}, \ldots, \tilde{\bs\varepsilon}^{(n)})$ be an independent copy of $\cE_n $. Applying Lemma~\ref{delapena-decop} yields the following decoupled bound of \eqref{I2-sum-U}
\begin{equation}\label{decomp-U2-ver}
I_2 \leq 12\; \mathbb{E}_{\cD_n}\cro{\mathbb{E}_{\cE_n,\widetilde\cE_n} \left[ \sup_{f \in \mathcal{F}_{\operatorname{NN}}} \left|\frac{1}{nm(m-1)}\sum_{i=1}^n\sum_{j \neq k} \kappa_f(X_i,{\bs\varepsilon}_j^{(i)}, \tilde{\bs\varepsilon}_k^{(i)}) \right| \Big|\cD_n\right]}.
\end{equation}
The term inside the supremum in \eqref{decomp-U2-ver} is an average over all off-diagonal pairs. The set of off-diagonal indices can be partitioned into $m-1$ cyclic permutations. By the convexity of the supremum norm and the identical distribution of these permutations, the expected supremum over the full off-diagonal sum is bounded by the expected supremum over a single diagonal sum. This reduces the problem to an empirical average of $m$ independent variables
\[
I_2\leq12\; \mathbb{E}_{\cD_n}\cro{\mathbb{E}_{\cE_n,\widetilde\cE_n} \left[ \sup_{f \in \mathcal{F}_{\operatorname{NN}}} \left|\frac{1}{nm}\sum_{i=1}^n\sum_{j=1}^{m} \kappa_f(X_i,{\bs\varepsilon}_j^{(i)}, \tilde{\bs\varepsilon}_j^{(i)}) \right| \Big|\cD_n\right]}.
\]

By doing so, conditionally on $\{X_i\}_{i=1}^n$, the term $\sum_{i=1}^n\sum_{j=1}^m \kappa_f(X_i,{\bs\varepsilon}_j^{(i)}, \tilde{\bs\varepsilon}_j^{(i)})$ is a sum of $mn$ independent zero-mean random variables. Applying standard symmetrization yields
\begin{equation}\label{I2-symm-b}
I_2\leq 24\; \mathbb{E}_{\cD_n}\cro{\mathbb{E}_{\cE_n,\widetilde\cE_n,\bm{\sigma}} \left[ \sup_{f \in \mathcal{F}_{\operatorname{NN}}} \left|\frac{1}{nm}\sum_{i=1}^n\sum_{j=1}^{m}\sigma_j^{(i)}\psi_f(X_i,{\bs\varepsilon}_j^{(i)}, \tilde{\bs\varepsilon}_j^{(i)}) \right| \Big|\cD_n\right]}, \end{equation}
where $\bm{\sigma} = (\bm{\sigma}_{1}, \ldots, \bm{\sigma}_{m})^{\top}$, and each $\bm{\sigma}_{j}=(\sigma^{(1)}_j,\ldots,\sigma^{(n)}_j)^{\top}$ consisting of $n$ i.i.d. Rademacher variables. 

Define the set of vectors $$\mathcal{U}_{|\cD_n,\cE_n,\widetilde{\cE}_n} =\left\{\big(\psi_f(X_1,\bm{\varepsilon}_1^{(1)},\widetilde{\bm{\varepsilon}}_1^{(1)}),\ldots,\psi_f(X_n,\bm{\varepsilon}_m^{(n)},\widetilde{\bm{\varepsilon}}_m^{(n)})\big): f \in \mathcal{F}_{\operatorname{NN}}\right\}$$ and define its symmetric completion as ${\mathcal{\check{U}}_{|\cD_n,\cE_n,\widetilde{\cE}_n}}=\mathcal{U}_{|\cD_n,\cE_n,\widetilde{\cE}_n}\cup\left\{-{\bs\gamma}:{\bs\gamma}\in\mathcal{U}_{|\cD_n,\cE_n,\widetilde{\cE}_n}\right\}$. It follows that for any $\delta>0$,
\begin{equation}\label{cover-bound-monte-I2}
\coverN(\delta, {\mathcal{\check{U}}_{|\cD_n,\cE_n,\widetilde{\cE}_n}}, \|\cdot\|_2) \leq2\coverN(\delta, \mathcal{U}_{|\cD_n,\cE_n,\widetilde{\cE}_n}, \|\cdot\|_2). 
\end{equation}
Using the notation introduced in \eqref{comp-def}, we denote ${\bs\gamma}= (\bm{\gamma}_{1}, \ldots, \bm{\gamma}_{m})^{\top}$, where each $\bm{\gamma}_{j}=(\gamma^{(1)}_j,\ldots,\gamma^{(n)}_j)^{\top}$. We derive from \eqref{I2-symm-b} that\vspace{-5pt}
\begin{align}
I_2&\leq24\mathbb{E}_{\cD_n,\cE_n,\widetilde\cE_n}\cro{\mathbb{E}_{\bm{\sigma}} \left[ \sup_{{\bs\gamma}\in {\mathcal{\check{U}}_{|\cD_n,\cE_n,\widetilde\cE_n}}}\frac{1}{nm}\sum_{i=1}^n\sum_{j=1}^{m}\sigma_j^{(i)}{\gamma}_j^{(i)}\right]}\nonumber\\
&=24\E_{\cD_n,\cE_n,\widetilde\cE_n}\cro{\mathfrak{R}\left(\check{\cU}_{|\cD_n,\cE_n,\widetilde\cE_n}\right)}.  \label{connect-I2-rader}
\end{align}

Applying Lemma~\ref{dudly-int} to the set $\check{\cU}_{|\cD_n,\cE_n,\widetilde\cE_n}$ and using \eqref{psi-lip-I2}, we can show by an argument similar to that used for $I_1$. Specifically,
\begin{equation*} 
\mathfrak{R}(\check{\cU}_{|\cD_n,\cE_n,\widetilde\cE_n})\leq\inf_{0<\delta<2\cB}\left(4\delta+\frac{24}{\sqrt{mn}}\int_{\delta}^{2\cB}\sqrt{\log\cro{2\coverN\big(\tau/2,\cF_{\operatorname{NN}},\|\cdot\|_{\infty}\big)}}\, d\tau\right),
\end{equation*}
which together with \eqref{connect-I2-rader} implies that
\begin{equation}\label{bound-monte-I2}
I_2 \leq 24\inf_{0<\delta<2\cB}\left(4\delta+\frac{24}{\sqrt{mn}}\int_{\delta}^{2\cB}\sqrt{\log\cro{2\coverN\big(\tau/2,\cF_{\operatorname{NN}},\|\cdot\|_{\infty}\big)}}\, d\tau\right).    
\end{equation}

Combining the preceding bounds \eqref{bound-monte-I1}, \eqref{bound-monte-I2} and substituting them into \eqref{total-dec-monte}, we finally obtain that
\begin{align*}
&\mathbb{E}_{\cD_n,\cE_n}\cro{\sup_{f\in\cF_{\operatorname{NN}}}\left|\frac{1}{n}\sum_{i=1}^n \cro{ \mathcal{L}(f; Z_i) - \widehat{\mathcal{L}}_m(f; Z_i, \bm{\varepsilon}^{(i)})}\right|}\\
&\leq I_1 + \frac{1}{2}I_2\\
&\leq14\inf_{0<\delta<2\cB}\left(4\delta+\frac{24}{\sqrt{mn}}\int_{\delta}^{2\cB}\sqrt{\log\cro{2\coverN\big(\tau/2,\cF_{\operatorname{NN}},\|\cdot\|_{\infty}\big)}}\, d\tau\right).
\end{align*}
\end{proof}

\begin{proof}[Proof of Theorem~\ref{thm-rate}]
The proof follows by combining Proposition~\ref{cover-bound} and Lemma~\ref{dnn-approx-lemma} within the framework of Theorem~\ref{thm:expect-main-m}. Specifically, setting $\delta_1 = \delta_2 = 2/n$ in the right-hand side of Theorem~\ref{thm:expect-main-m}, we obtain that
\begin{align}
&\mathbb{E}_{\cD_n,\cE_n}\bigl[\cR(\widehat f_m)-\cR(f^*)\bigr]\nonumber\\&\leq C'\cro{\inf_{f\in\mathcal{F}(L_n,\bs{p}_n, s_n,\cB)}\|f-f^*\|_{L^{\infty}(\cX\times\cE)}+\sqrt{\frac{\log\cN\big(1/n,\mathcal{F}(L_n,\bs{p}_n, s_n,\cB),\|\cdot\|_{\infty}\big)}{n}}},\label{rate-b-1}  \end{align}
where $C'>0$ is a numerical constant. Applying Lemma~\ref{dnn-approx-lemma} and Proposition~\ref{cover-bound} to the right-hand side of \eqref{rate-b-1} yields
\begin{align*}
\mathbb{E}_{\cD_n,\cE_n}\bigl[\cR(\widehat f_m)-\cR(f^*)\bigr]\leq C'\max_{i=0,\ldots,q}\cro{Cn^{-\frac{\alpha_i^*}{2\alpha_i^*+t_i}}+C''n^{-\frac{\alpha_i^*}{2\alpha_i^*+t_i}}(\log n)^{3/2}},   
\end{align*}
which completes the proof.
\end{proof}

\subsection{Performance under the original energy loss}
In this subsection, we present the theoretical performance guarantees for the estimator $\widehat f_{\infty}$, using an analysis that parallels the argument for $\widehat f_m$. Recall that $\mathcal{F}_{\operatorname{NN}}$ denotes a class of ReLU networks mapping from $\mathcal{X} \times \mathcal{E}$ to $\mathcal{Y}$. The next result shows a decomposition of the excess risk of $\widehat f_{\infty}$ into stochastic and approximation errors.
\begin{lemma}\label{decom-1}
Let $\widehat f_{\infty}$ be the estimator defined in \eqref{loss-1} over the class $\cF_{\operatorname{NN}}$ of functions mapping $\cX\times\cE$ to $\cY$. Then, for any $f \in \mathcal{F}_{\operatorname{NN}}$,
\[
\mathcal{R}(\widehat f_{\infty}) - \mathcal{R}(f^*) \leq 2\sup_{g \in \mathcal{F}_{\operatorname{NN}}}\big|\mathcal{R}(g)-\mathcal{R}_n(g)\big| + \mathcal{R}(f)-\mathcal{R}(f^*).
\]  
\end{lemma}
\begin{proof}\label{proof-decom-1}
Since $\widehat f_{\infty}$ is the empirical risk minimizer over $\cF_{\operatorname{NN}}$, we have $\mathcal{R}_n(\widehat f_{\infty})\leq\mathcal{R}_n(f)$ for any fixed $f \in \mathcal{F}_{\operatorname{NN}}$. Therefore,
\begin{align*}
&\mathcal{R}(\widehat f_{\infty})-   \mathcal{R}(f^*)\\
&=\mathcal{R}(\widehat f_{\infty})-\mathcal{R}_n(\widehat f_{\infty})+\mathcal{R}_n(\widehat f_{\infty})-\mathcal{R}_n(f)+\mathcal{R}_n(f)- \mathcal{R}(f)  +\mathcal{R}(f)-\mathcal{R}(f^*)\\
&\leq\mathcal{R}(\widehat f_{\infty})-\mathcal{R}_n(\widehat f_{\infty})+\mathcal{R}_n(f)- \mathcal{R}(f)  +\mathcal{R}(f)-\mathcal{R}(f^*)\\
&\leq2\sup_{g\in\mathcal{F}_{\operatorname{NN}}}\big|\mathcal{R}(g)-\mathcal{R}_n(g)\big|+\mathcal{R}(f)-\mathcal{R}(f^*).
\end{align*}
\end{proof}
Next, we bound the stochastic error term in Lemma~\ref{decom-1} using the Rademacher complexity of the network class $\cF_{\operatorname{NN}}$. 
\begin{lemma}\label{stochastic-error}
Suppose there exists a finite constant $\cB$ such that $\|f\|_{L^\infty(\cX\times\cE)}\leq \cB$ for all $f \in \cF_{\operatorname{NN}} \cup \{f^*\}$. Then, 
\begin{align*}
\mathbb{E}_{\cD_n}\left[\sup_{f \in \mathcal{F}_{\operatorname{NN}}} \bigl| \mathcal{R}(f) - {\mathcal{R}}_n(f) \bigr| \right]\leq\inf_{0<\delta<2\cB}\left(8\delta+\frac{24}{\sqrt{n}}\int_{\delta}^{2\cB}\hspace{-8pt}\sqrt{\log\cro{2\coverN\big(\tau/2,\cF_{\operatorname{NN}},\|\cdot\|_{\infty}\big)}}\, d\tau\right).   
\end{align*}
\end{lemma}
\begin{proof}
The proof follows an argument similar to that of Lemma~\ref{lem:emp-stoch-error}; hence, we only provide a sketch here. Let $\mathcal{D}'_n = \{(X_i', Y_i')\}_{i=1}^n$ be an independent copy of $\mathcal{D}_n = \{(X_i, Y_i)\}_{i=1}^n$. It follows that
\begin{align}
&\mathbb{E}_{\cD_n}\left[\sup_{f \in \mathcal{F}_{\operatorname{NN}}} \bigl| \mathcal{R}(f) - {\mathcal{R}}_n(f) \bigr| \right]\nonumber\\
&= \mathbb{E}_{\cD_n} \left[ \sup_{f \in \mathcal{F}_{\operatorname{NN}}}  \left| \mathbb{E}_{\cD'_n} \cro{ \frac{1}{n}\sum_{i=1}^n \cL(f; X_i', Y_i')} - \frac{1}{n}\sum_{i=1}^n \cL(f; X_i, Y_i) \right| \right]\nonumber\\
&\leq \mathbb{E}_{\cD_n,\cD'_n} \left[ \sup_{f \in \mathcal{F}_{\operatorname{NN}}} \Bigl| \frac{1}{n}\sum_{i=1}^n \bigl( \cL(f; X_i', Y_i') - \cL(f; X_i, Y_i) \bigr) \Bigr| \right].\label{sym-bound}
\end{align}

Let $\sigma_1, \dots, \sigma_n$ be i.i.d. Rademacher random variables, where $\mathbb{P}(\sigma_i = 1) = \mathbb{P}(\sigma_i = -1) = 1/2$. These variables are independent of the data $\mathcal{D}_n$ and $\mathcal{D}'_n$. Define $\boldsymbol{\sigma} = (\sigma_1, \dots, \sigma_n)^{\top}$. Swapping $(X_i,Y_i)$ and $(X_i',Y_i')$ does not change the distribution, hence
\begin{align}
&\mathbb{E}_{\cD_n,\cD'_n}\!\left[ \sup_{f \in \mathcal{F}_{\operatorname{NN}}} \Bigl| \frac{1}{n}\sum_{i=1}^n \bigl( \cL(f; X_i', Y_i') - \cL(f; X_i, Y_i) \bigr) \Bigr| \right]\nonumber\\
&= \mathbb{E}_{\cD_n,\cD'_n,{\bs\sigma}}\!\left[ \sup_{f \in \mathcal{F}_{\operatorname{NN}}}  \Bigl| \frac{1}{n}\sum_{i=1}^n \sigma_i \bigl( \cL(f; X_i', Y_i') - \cL(f; X_i, Y_i) \bigr) \Bigr| \right]\nonumber\\  &\leq 2 \, \mathbb{E}_{\cD_n,{\bs{\sigma}}}\!\left[ \sup_{f \in \mathcal{F}_{\operatorname{NN}}}\Bigl| \frac{1}{n}\sum_{i=1}^n \sigma_i \cL(f; X_i, Y_i) \Bigr| \right].\label{sym-b} 
\end{align}

Define $$\mathcal{G} = \{ (x, y) \mapsto \cL(f; x, y) : f \in \mathcal{F}_{\operatorname{NN}} \}$$ as the class of loss functions induced by $\mathcal{F}_{\operatorname{NN}}$ and let $-\cG=\{-g: g\in\cG\}$. It follows from \eqref{sym-bound} and \eqref{sym-b} that
\begin{equation}\label{union-bound}
\mathbb{E}_{\cD_n}\left[\sup_{f \in \mathcal{F}_{\operatorname{NN}}} \bigl| \mathcal{R}(f) - {\mathcal{R}}_n(f) \bigr| \right]\leq  2\mathbb{E}_{\cD_n,{\bs{\sigma}}}\!\left[ \sup_{g \in \mathcal{G} \cup -\mathcal{G}} \frac{1}{n}\sum_{i=1}^n \sigma_i g(X_i, Y_i)\right].  
\end{equation}

We now bound the right-hand side of \eqref{union-bound}. For any set $\cM \subseteq \mathbb{R}^n$, define the Rademacher complexity of $\cM$ as
\begin{equation*}
\mathfrak{R}(\cM) = \mathbb{E}_{{\bs{\sigma}}}\left[\sup_{(\zeta_1,\ldots,\zeta_n) \in \cM} \frac{1}{n}\sum_{i=1}^n \sigma_i \zeta_i \right].
\end{equation*}
Conditionally on the sample $\cD_n = \{Z_1, \dots, Z_n\}$ with $Z_i=(X_i,Y_i)$, we denote the set of loss vectors by
\begin{equation*}
\cF_{|\cD_n} = \left\{ \big(\cL(f; Z_1), \ldots, \cL(f; Z_n)\big) : f \in \cF_{\operatorname{NN}} \right\}.
\end{equation*}
To account for the union class $\mathcal{G} \cup -\mathcal{G}$, we define the symmetric completion of this set as $\check{\cF}_{|\cD_n} = \cF_{|\cD_n} \cup \{-\bs{\zeta} : \bs{\zeta} \in \cF_{|\cD_n}\}$. Consequently, \eqref{union-bound} can be rewritten as
\begin{equation}\label{rad-comp-final}
\mathbb{E}_{\cD_n}\left[\sup_{f \in \mathcal{F}_{\operatorname{NN}}} \bigl| \mathcal{R}(f) - {\mathcal{R}}_n(f) \bigr| \right] \leq 2 \, \mathbb{E}_{\cD_n} \big[ \mathfrak{R}(\check{\cF}_{|\cD_n}) \big].
\end{equation}

Observe that for any $f\in\cF_{\operatorname{NN}}$ and any fixed $(x,y)$, Proposition~\ref{bound-value-cL} shows
\begin{equation}\label{bound-for-cL}
|\cL(f;x,y)|\leq2\cB.    
\end{equation}
Recall that for any given two vectors ${\bs u}, {\bs v}\in\R^n$, 
$$\|{\bs u}-{\bs v}\|_2=\sqrt{\frac{1}{n}\sum_{i=1}^n(u_i-v_i)^2}.$$ Denote the $\delta$-covering number of the set $\check{\cF}_{|\cD_n}\subseteq\R^n$ with respect to the normalized Euclidean distance $\|\cdot\|_2$ by $\coverN(\delta,\check{\cF}_{|\cD_n},\|\cdot\|_2)$. Since any $\delta$-cover of $\cF_{|\cD_n}$ can be mirrored to form a $\delta$-cover of $-\cF_{|\cD_n}$, the covering numbers satisfy $\mathcal{N}(\delta, \check{\cF}_{|\cD_n}, \|\cdot\|_2) \leq 2\mathcal{N}(\delta, \cF_{|\cD_n}, \|\cdot\|_2)$. Applying Lemma~\ref{dudly-int} to the set $\check{\cF}_{|\cD_n}$ and using \eqref{bound-for-cL}, we obtain that for any nonnegative integer $M$,
\begin{align*}
\mathfrak{R}(\check{\cF}_{|\cD_n})&\leq 2^{-M+2}\cB+\frac{24\cB}{\sqrt{n}}\sum_{m=1}^{M}2^{-m}\sqrt{\log\cN\big(2^{-m+2}\cB,\check{\cF}_{|\cD_n},\|\cdot\|_2\big)}\\
&\leq 2^{-M+2}\cB + \frac{12}{\sqrt{n}} \int_{2^{-M+1}\cB}^{2\cB} \sqrt{\log\cro{2\mathcal{N} \big( \tau,\cF_{|\cD_n}, \|\cdot\|_2 \big)}} \, d\tau.
\end{align*} 
Observe that for any $\delta\in(0,2\cB)$, there exists an integer $M$ satisfying $2^{-M}\cB\leq\delta<2^{-M+1}\cB$. It then follows that
\begin{equation}\label{bb-1}
\mathfrak{R}(\check{\cF}_{|\cD_n})\leq\inf_{0<\delta<2\cB}\left(4\delta+\frac{12}{\sqrt{n}}\int_{\delta}^{2\cB}\sqrt{\log\cro{2\coverN\big(\tau,\cF_{|\cD_n},\|\cdot\|_2\big)}}\, d\tau\right).
\end{equation}
Moreover, note that for any two functions $f,g\in\cF_{\operatorname{NN}}$ satisfying $\|f-g\|_{L^{\infty}(\cX\times\cE)}\leq\delta/2$, Proposition~\ref{loss-Lip-2} implies that almost surely
\begin{align*}
\left|\cL(f;X,Y)-\cL(g;X,Y)\right|\leq2\E_{\varepsilon}\cro{\|f(X,\varepsilon)-g(X,\varepsilon)\|_2}\leq\delta.
\end{align*}
Thus, we have
\begin{equation}\label{l2-sup-cover-conne}
\log\coverN(\delta,\cF_{|\cD_n},\|\cdot\|_2)\leq  \log\coverN(\delta/2,\cF_{\operatorname{NN}},\|\cdot\|_{\infty}).
\end{equation}
Substituting the covering number bound from \eqref{l2-sup-cover-conne} into inequality \eqref{bb-1} implies, almost surely, that\vspace{-10pt}
\begin{equation}\label{rader-bb}
\mathfrak{R}(\check{\cF}_{|\cD_n})\leq\inf_{0<\delta<2\cB}\left(4\delta+\frac{12}{\sqrt{n}}\int_{\delta}^{2\cB}\sqrt{\log\cro{2\coverN\big(\tau/2,\cF_{\operatorname{NN}},\|\cdot\|_{\infty}\big)}}\, d\tau\right).    
\end{equation}
Taking the expectation over the sample $\cD_n=\{(X_i,Y_i)\}_{i=1}^n$ on both sides of \eqref{rader-bb} and combining \eqref{rad-comp-final} yields
\begin{align*}
&\mathbb{E}_{\cD_n}\left[\sup_{f \in \mathcal{F}_{\operatorname{NN}}} \bigl| \mathcal{R}(f) - {\mathcal{R}}_n(f) \bigr| \right]\\ &\leq\inf_{0<\delta<2\cB}\left(8\delta+\frac{24}{\sqrt{n}}\int_{\delta}^{2\cB}\sqrt{\log\cro{2\coverN\big(\tau/2,\cF_{\operatorname{NN}},\|\cdot\|_{\infty}\big)}}\, d\tau\right). 
\end{align*}
\end{proof}
With the results established above, we obtain the convergence rate of the estimator defined in \eqref{loss-1} based on deep sparse neural networks.
\begin{theorem}
Suppose $f^{*}\in\cG(q,{\bs{d}},{\bs{t}},{\bm{\alpha}},K)$ with $K\geq1$. Set $N_i=\lceil n^{t_i/(t_i+2\alpha_i^*)}\rceil$, $N=\max_{i=0,\ldots,q}N_i$, and let $\widehat f_{\infty}$ denote the estimator obtained from \eqref{loss-1} over the network class $\mathcal{F}(L_n,\bs{p}_n, s_n,\cB)$, whose architecture satisfies
\begin{enumerate}[label=(\roman*), itemsep=0pt, topsep=2pt, parsep=0pt]
\item $L_n=C_{1} \log_2 n$;
\item ${\bs{p}}_n=\left(d+k,C_{2}N,\ldots,C_{2}N,p\right)$;
\item $s_n=C_{3} N\log_2 n$;
\item $\cB\geq K$.
\end{enumerate}
Then, for all sufficiently large $n$, 
\begin{align*}
\mathbb{E}_{\cD_n}\bigl[\cR(\widehat f_{\infty})-\cR(f^*)\bigr]\leq C_{5}\max_{i=0,\ldots,q}n^{-\frac{\alpha_i^*}{2\alpha_i^*+t_i}}\log^{3/2} n.
\end{align*} 
Here, $C_1,C_2,C_3$ are the constants stated in Theorem~\ref{thm-rate}, and $C_5$ is a positive numerical constant independent of $n$.    
\end{theorem}
\begin{proof}
The proof is similar to that of Theorem~\ref{thm-rate}. Combining \eqref{approx-error}, Lemma~\ref{decom-1} and Lemma~\ref{stochastic-error}, we have
\begin{align*}
&\mathbb{E}_{\cD_n}\cro{ \mathcal{R}(\widehat f_{\infty}) - \mathcal{R}(f^*)}\\
&\leq 2\inf_{f\in\mathcal{F}_{\mathrm{NN}}}\hspace{-3pt}\|f-f^*\|_{L^\infty(\cX\times\cE)}+16\inf_{0<\delta<2\cB}\left(\delta+\frac{3}{\sqrt{n}}\int_{\delta}^{2\cB}\hspace{-7pt}\sqrt{\log\cro{2\coverN\big(\tau/2,\cF_{\operatorname{NN}},\|\cdot\|_{\infty}\big)}}\, d\tau\right). \end{align*}
Then, by setting $\delta=2/n$ and applying Proposition~\ref{cover-bound} and Lemma~\ref{dnn-approx-lemma}, we complete the proof.
\end{proof}
Comparing the above result with Theorem~\ref{thm-rate} in the main text, we conclude that, after taking expectation with respect to the sampled noise variables, the estimator $\widehat f_m$ achieves the same convergence rate as the ideal estimator $\widehat f_{\infty}$.

\bibliographystyle{apalike}
\bibliography{bib}  
\end{document}

%% file: begin0.tex
\usepackage[utf8]{inputenc}
\usepackage[T1]{fontenc}
\usepackage{lmodern}
\usepackage{yhmath}

\usepackage{amsthm,amsmath,amssymb,amsbsy,bbm,mathrsfs,supertabular,
eurosym,graphicx,enumitem,xcolor}
\usepackage{mathtools}

\newtheoremstyle{BBstyle0}  {}{}{\itshape}{}{\bfseries}{}{6pt}{}
\newtheoremstyle{BBstyle1}  {3pt}{3pt}{\rmfamily}{}{\itshape}{: }{3pt}{}
\newtheoremstyle{BBstyle2}  {3pt}{3pt}{\itshape}{}{\bfseries\large}{}{0pt}{}
\newtheoremstyle{BBstyle3}  {}{}{\itshape}{}{\bfseries}{: }{3pt}{}
\newtheoremstyle{BBstyle4}  {}{}{\rmfamily}{}{\bfseries}{}{6pt}{}
  
\usepackage[normalem]{ulem}

%% file: core.tex
\newcommand{\cro}[1]{\left[{#1}\right]}

\newcommand{\argmin}{\mathop{\rm argmin}}

\newcommand{\R}{{\mathbb{R}}}

\DeclareMathAlphabet{\mathscrbf}{OMS}{mdugm}{b}{n}

\newcommand{\cA}{{\mathcal{A}}}
\newcommand{\cB}{{\mathcal{B}}}

\newcommand{\cD}{{\mathcal{D}}}
\newcommand{\cE}{{\mathcal{E}}}
\newcommand{\cF}{{\mathcal{F}}}
\newcommand{\cG}{{\mathcal{G}}} 
\newcommand{\cH}{{\mathcal{H}}}

\newcommand{\cL}{{\mathcal{L}}} 
\newcommand{\cM}{{\mathcal{M}}}
\newcommand{\cN}{{\mathcal{N}}}

\newcommand{\cR}{{\mathcal{R}}}

\newcommand{\cU}{{\mathcal{U}}}
\newcommand{\cV}{{\mathcal{V}}}

\newcommand{\cX}{{\mathcal{X}}}
\newcommand{\cY}{{\mathcal{Y}}}

\newcommand{\bs}[1]{\boldsymbol{#1}}

%% file: bib.bib
@article{Schmidt-Hieber,
author = {Johannes Schmidt-Hieber},
title = {{Nonparametric regression using deep neural networks with ReLU activation function}},
journal = {Annals of Statistics},
year = {2020},
volume = {48},
number = {4},
pages = {1875-1897},
}

@book{shalevshwartz2014understanding,
  title     = {Understanding Machine Learning: From Theory to Algorithms},
  author    = {Shalev-Shwartz, Shai and Ben-David, Shai},
  year      = {2014},
  month     = {07},
  day       = {05},
  publisher = {Cambridge University Press},
  address   = {Cambridge},
  isbn      = {9781107298019}
}

@article{baringhaus2004new,
  author  = {Baringhaus, Ludwig and Franz, Claudia},
  title   = {On a new multivariate two-sample test},
  journal = {Journal of Multivariate Analysis},
  year    = {2004},
  volume  = {88},
  number  = {1},
  pages   = {190--206}
}

@article{shen2024engression,
    author = {Shen, Xinwei and Meinshausen, Nicolai},
    title = {Engression: extrapolation through the lens of distributional regression},
    journal = {Journal of the Royal Statistical Society Series B: Statistical Methodology},
    volume = {87},
    number = {3},
    pages = {653-677},
    year = {2024},
    month = {11},
    issn = {1369-7412},
    doi = {10.1093/jrsssb/qkae108},
    url = {https://doi.org/10.1093/jrsssb/qkae108},
    eprint = {https://academic.oup.com/jrsssb/article-pdf/87/3/653/60827977/qkae108.pdf},
}

@book{goodfellow2016deep,
  author    = {Goodfellow, Ian and Bengio, Yoshua and Courville, Aaron},
  title     = {Deep Learning},
  publisher = {MIT Press},
  year      = {2016}
}

@article{peter1,
  author={Bartlett, P. and Maiorov, V. and Meir, R.},
  journal={Neural Computation}, 
  title={Almost Linear {VC} Dimension Bounds for Piecewise Polynomial Networks}, 
  year={1998},
  volume={10},
  number={8},
  pages={2159--2173}
}

@article{peter2,
    author = "Bartlett, P. and Harvey, N. and Liaw, C. and Mehrabian, A.",
    title ="Nearly-tight {VC}-dimension and Pseudodimension Bounds for Piecewise Linear Neural Networks" ,
    journal = "Journal of Machine Learning Research" ,
    year = "2019",
    volume="20",
    pages="1--17"
}

@article{2014dropout,
  author    = {Srivastava, Nitish and Hinton, Geoffrey and Krizhevsky, Alex and Sutskever, Ilya and Salakhutdinov, Ruslan},
  title     = {Dropout: A Simple Way to Prevent Neural Networks from Overfitting},
  journal   = {Journal of Machine Learning Research},
  volume    = {15},
  pages     = {1929--1958},
 year  = {2014},
}

@article{szekely2013energy,
  author    = {Sz{\'e}kely, G{\'a}bor J. and Rizzo, Maria L.},
  title     = {Energy statistics: A class of statistics based on distances},
  journal   = {Journal of Statistical Planning and Inference},
  year      = {2013},
  volume    = {143},
  number    = {8},
  pages     = {1249--1272}
}

@book{delapena,
  author    = {de la Peña, Victor H. and Giné, Evarist},
  title     = {Decoupling: From Dependence to Independence},
  series    = {Probability and Its Applications},
  publisher = {Springer},
  address   = {New York},
  year      = {1999},
  isbn      = {978-0-387-98616-6},
  pages     = {392}
}

@article{Bartlett_2005,
   title={Local {R}ademacher complexities},
   volume={33},
   number={4},
   journal={Annals of Statistics},
   publisher={Institute of Mathematical Statistics},
   author={Bartlett, Peter L. and Bousquet, Olivier and Mendelson, Shahar},
   pages={1497-1537},
   year={2005}
}

@book{van-exp,
  author    = {Aad W. van der Vaart and Jon A. Wellner},
  title     = {Weak Convergence and Empirical Processes: With Applications to Statistics},
  edition   = {Second},
  year      = {2023},
  publisher = {Springer},
  series    = {Springer Series in Statistics}
}

@inproceedings{sharma2023bayesian,
  author    = {Mudit Sharma and Sebastian Farquhar and Eric Nalisnick and Tom Rainforth},
  title     = {Do {B}ayesian Neural Networks Need to Be Fully Stochastic?},
  booktitle = {Proceedings of the 26th International Conference on Artificial Intelligence and Statistics},
  series    = {Proceedings of Machine Learning Research},
  pages     = {7694--7722},
  year      = {2023}
}

@article{Vu2024Correction,
  author  = {Schmidt-Hieber, Johannes and Vu, Don},
  title   = {Correction to ``{N}onparametric regression using deep neural networks with {ReLU} activation function''},
  journal = {Annals of Statistics},
  year    = {2024},
  volume  = {52},
  number  = {1},
  pages   = {413--414},
  note    = {Correction to the 2020 article}
}

@article{chai2026,
  title = {Neural generative distributional regression},
  author = {Chai, Jinhang and Fan, Jianqing and Gu, Yihong},
  year = {2026},
 journal = {arXiv preprint, arXiv:2605.02062}
}

@article{sophie2021,
  author  = {Kohler, Michael and Langer, Sophie},
  title   = {On the Rate of Convergence of Fully Connected Deep Neural Network Regression Estimates},
  journal = {Annals of Statistics},
  year    = {2021},
  volume  = {49},
  number  = {4},
  pages   = {2231--2249}
}

@article{li1991sliced,
  title={Sliced inverse regression for dimension reduction},
  author={Li, Ker-Chau},
  journal={Journal of the American Statistical Association},
  volume={86},
  number={414},
  pages={316--327},
  year={1991},
  publisher={Taylor \& Francis}
}

@article{michigan2026,
    author = {Jiaqi Huang and Gongjun Xu and Ji Zhu},
    title = {Theoretical Analysis of Engression and Reverse {M}arkov Engression},
    journal = {arXiv preprint, arXiv:2606.01002},
    year = {2026}
}

@article{pic2023,
  author    = {Romain Pic and Cl{\'e}ment Dombry and Philippe Naveau and Maxime Taillardat},
  title     = {Distributional regression and its evaluation with the {CRPS}: Bounds and convergence of the minimax risk},
  journal   = {International Journal of Forecasting},
  volume    = {39},
  number    = {4},
  pages     = {1564--1572},
  year      = {2023},
  doi       = {10.1016/j.ijforecast.2022.11.001},
  publisher = {Elsevier}
}

@article{peter2025,
      title={Generative modelling via quantile regression}, 
      author={Schmidt-Hieber, Johannes and Zamolodtchikov, Petr},
      year={2024},
journal={arXiv preprint, arXiv:2409.04231}
}

@book{cook1998regression,
  title={Regression Graphics: Ideas for Studying Regressions through Graphics},
  author={Cook, R. Dennis},
  year={1998},
  publisher={John Wiley \& Sons},
  address={New York}
}

@article{bengio2013representation,
  title={Representation learning: A review and new perspectives},
  author={Bengio, Yoshua and Courville, Aaron and Vincent, Pascal},
  journal={IEEE Transactions on Pattern Analysis and Machine Intelligence},
  volume={35},
  number={8},
  pages={1798--1828},
  year={2013},
  publisher={IEEE}
}

@book{koenker2005quantile,
  title={Quantile Regression},
  author={Koenker, Roger},
  year={2005},
  publisher={Cambridge University Press},
  address={Cambridge}
}

@article{hall2004cross,
  title={Cross-validation and the estimation of conditional probability densities},
  author={Hall, Peter and Racine, Jeffrey S and Li, Qi},
  journal={Journal of the American Statistical Association},
  volume={99},
  number={468},
  pages={1015--1026},
  year={2004},
  publisher={Taylor \& Francis}
}

@article{Huang2022GANerror,
  title   = {An Error Analysis of Generative Adversarial Networks for Learning Distributions},
  author  = {Huang, Jian and Jiao, Yuling and Li, Zhen and Liu, Shiao and Wang, Yang and Yang, Yunfei},
  journal = {Journal of Machine Learning Research},
  volume  = {23},
  pages   = {1--43},
  year    = {2022}
}

@article{Stone1982,
  title   = {Optimal global rates of convergence for nonparametric regression},
  author  = {Stone, Charles J.},
  journal = {Annals of Statistics},
  year    = {1982},
  volume  = {10},
  number  = {4},
  pages   = {1040--1053}
}

@book{Tsybakov2009,
  title     = {Introduction to Nonparametric Estimation},
  author    = {Tsybakov, Alexandre B.},
  year      = {2009},
  publisher = {Springer},
  address   = {New York}
}

@article{HuangHorowitzWei2010,
  title   = {Variable Selection in Nonparametric Additive Models},
  author  = {Huang, Jian and Horowitz, Joel L. and Wei, Feng},
  journal = {Annals of Statistics},
  year    = {2010},
  volume  = {38},
  number  = {4},
  pages   = {2282--2313}
}

@article{chen2026semi,
  title={Semi-Supervised Learning on Graphs using Graph Neural Networks},
  author={Chen, Juntong and Donnat, Claire and Klopp, Olga and Schmidt-Hieber, Johannes},
  journal={arXiv preprint arXiv:2602.17115},
  year={2026}
}

@inproceedings{goodfellow2014generative,
  author    = {Ian J. Goodfellow and Jean Pouget-Abadie and Mehdi Mirza and Bing Xu and David Warde-Farley and Sherjil Ozair and Aaron Courville and Yoshua Bengio},
  title     = {Generative Adversarial Nets},
  booktitle = {Advances in Neural Information Processing Systems},
  volume    = {27},
  pages     = {2672--2680},
  year      = {2014}
}

@article{zhou2023deep,
  author    = {Xingyu Zhou and Yuling Jiao and Jin Liu and Jian Huang},
  title     = {A Deep Generative Approach to Conditional Sampling},
  journal   = {Journal of the American Statistical Association},
  volume    = {118},
  number    = {543},
  pages     = {1837--1848},
  year      = {2023},
  doi       = {10.1080/01621459.2021.2016424}
}

@article{austin2015,
  author    = {Tim Austin},
  title     = {Exchangeable Random Measures},
  journal   = {Annales de l'Institut Henri Poincar\'e, Probabilit\'es et Statistiques},
  volume    = {51},
  number    = {3},
  pages     = {842--861},
  year      = {2015},
  doi       = {10.1214/13-AIHP584}
}

@article{song2026wasserstein,
  author = {Song, Shanshan and Wang, Tong and Shen, Guohao and Lin, Yuanyuan and Huang, Jian},
  title = {Wasserstein generative regression},
  journal = {Journal of the Royal Statistical Society Series B: Statistical Methodology},
  year = {2026},
  volume = {88},
  number = {1},
  pages = {330--351},
  doi = {10.1093/jrsssb/qkaf053},
  publisher = {Oxford University Press}
}

@book{kallenberg,
  author    = {Olav Kallenberg},
  title     = {Foundations of Modern Probability},
  series    = {Probability and Its Applications},
  year      = {1997},
  publisher = {Springer},
  address   = {New York},
  edition   = {1st},
  doi       = {10.1007/978-1-4757-4015-8}
}

@article{2017regularity,
  title={A regularity class for the roots of nonnegative functions},
  author={Ray, Kolyan and Schmidt-Hieber, Johannes},
  journal={Annali di Matematica Pura ed Applicata},
  year={2017},
  volume={196},
  number={6},
  pages={2101--2113}
}

@article{baraud2014estimating,
  title={Estimating composite functions by model selection},
  author={Baraud, Yannick and Birg{\'e}, Lucien},
  journal={Annales de l'Institut Henri Poincar\'e, Probabilit\'es et Statistiques},
  volume={50},
  number={1},
  pages={285--314},
  year={2014}
}

@article{juditsky2009nonparametric,
  title={Nonparametric estimation of composite functions},
  author={Juditsky, Anatoli and Lepski, Oleg and Tsybakov, Alexandre},
  journal={Annals of Statistics},
  volume={37},
  number={3},
  pages={1360--1404},
  year={2009}
}
